\documentclass[12pt,dvipsnames]{colt2019} 

\title[Pure Exploration with Multiple Correct Answers]{%
  Pure Exploration with Multiple Correct Answers
}
\usepackage{times}
\usepackage{hyperref,amsfonts}
\usepackage{bm}
\let\set\undefined 
\usepackage{notation}
\usepackage{todonotes}
\usetikzlibrary{shapes.geometric}
\usepackage{algorithm}
\usepackage{enumitem,array,booktabs}
\usepackage{wrapfig}

\DeclareMathOperator{\conv}{conv} 
\DeclareMathOperator*{\argmax}{argmax}
\DeclareMathOperator*{\argmin}{argmin}
\DeclareMathOperator{\ex}{\mathbb E}
\DeclareMathOperator{\pr}{\mathbb P}
\DeclareMathOperator{\qr}{\mathbb Q}
\DeclareMathOperator{\R}{\mathbb R}
\DeclareMathOperator{\N}{\mathbb N}
\DeclareMathOperator{\KL}{KL}

\let\inf\undefined
\let\min\undefined
\let\max\undefined
\DeclareMathOperator*{\inf}{\vphantom{sup}inf}
\DeclareMathOperator*{\min}{\vphantom{sup}min}
\DeclareMathOperator*{\max}{\vphantom{sup}max}

\DeclareBoldMathCommand{\vmu}{\mu}
\DeclareBoldMathCommand{\vlambda}{\lambda}
\DeclareBoldMathCommand{\vtheta}{\theta}
\DeclareBoldMathCommand{\w}{w}
\DeclareBoldMathCommand{\q}{q}
\DeclareBoldMathCommand{\p}{p}
\DeclareBoldMathCommand{\e}{e}
\DeclareBoldMathCommand{\a}{a}
\DeclareBoldMathCommand{\u}{u}
\let\top\intercal 
\newcommand{\ihat}{\hat \imath}

\coltauthor{\Name{R{\'e}my Degenne} \Email{remy.degenne@cwi.nl}\\
  \Name{Wouter M. Koolen} \Email{wmkoolen@cwi.nl}\\
  \addr Centrum Wiskunde \& Informatica, Science Park 123, Amsterdam, NL}

\begin{document}

\maketitle

\begin{abstract}%
  We determine the sample complexity of pure exploration bandit problems with multiple good answers. We derive a lower bound using a new game equilibrium argument. We show how continuity and convexity properties of \emph{single-answer} problems ensures that the Track-and-Stop algorithm has asymptotically optimal sample complexity. However, that convexity is lost when going to the \emph{multiple-answer} setting. We present a new algorithm which extends Track-and-Stop to the multiple-answer case and has asymptotic sample complexity matching the lower bound.
\end{abstract}




\section{Introduction}


In \emph{pure exploration} aka \emph{active testing} problems the learning system interacts with its environment by sequentially performing experiments to quickly and reliably identify the answer to a particular pre-specified question.
Practical applications range from simple queries for cost-constrained physical regimes, i.e.\ clinical drug testing, to complex queries in structured environments bottlenecked by computation, i.e.\ simulation-based planning.
The theory of pure exploration is studied in the multi-armed bandit framework.
The scientific challenge is to develop tools for characterising the sample complexity of new pure exploration problems, and methodologies for developing (matching) algorithms.
With the aim of understanding power and limits of existing methodology, we study an extended problem formulation where each instance may have multiple correct answers. We find that multiple-answer problems present a phase transition in complexity, and require a change in our thinking about algorithms.

The existing methodology for developing asymptotically instance-optimal algorithms, Track-and-Stop by \citet{GK16}, exploits the so-called \emph{oracle weights}. These probability distributions on arms naturally arise in sample complexity lower bounds, and dictate the optimal sampling proportions for an ``oracle'' algorithm that needs to be successful only on exactly the current problem instance. The main idea is to track the oracle weights at a converging estimate of the instance.
The analysis of Track-and-Stop requires continuity of the oracle weights as a function of the bandit model. For the core Best Arm Identification problem, discontinuity only occurs at degenerate instances where the sample complexity explodes. So this assumption seemed harmless.

\paragraph{Our contribution}
We show that the oracle weights may be non-unique, even for single-answer problems, and hence need to be regarded as a set-valued mapping. We show this mapping is always (upper hemi-)continuous. We show that it is convex for single-answer problems, and this allows us to extend the Track-and-Stop methodology to all such problems. At instances with non-singleton set-valued oracle weights more care is needed: of the two classical tracking schemes ``C'' converges to the convex set, while ``D'' may fail entirely.

We show that for multiple-answer problems convexity is violated. There are instances where two distinct oracle weights are optimal, while no mixture is. Unmodified tracking converges in law (experimentally) to a distribution on the full convex hull, and suffers as a result. We propose a ``sticky'' modification to stabilise the approach, and show that now it converges to only the corners. We prove that Sticky Track-and-Stop is asymptotically optimal.

\paragraph{Related work}

Multi-armed bandits have been the subject of intense study in their role as a model for medical testing and reinforcement learning. For the objective of reward maximisation \citep{Berry:Fristedt85,LaiRobbins85bandits,Aueral02,Bubeck:Survey12} the main challenge is balancing exploration and exploitation. The field of pure exploration (active testing) focuses on generalisation vs sample complexity, in the fixed confidence, fixed budget and simple regret scalarisations. It took off in machine learning with the \emph{multiple-answer} problem of $(\epsilon,\delta)$-Best Arm Identification (BAI) \citep{DBLP:conf/colt/Even-DarMM02}. Early results focused on worst-case sample complexity guarantees in sub-Gaussian bandits. Successful approaches include
\emph{Upper and Lower confidence bounds} \citep{Bubeckal11,Shivaramal12,Gabillon:al12,COLT13,Jamiesonal14LILUCB},
\emph{Racing or Successive Rejects/Eliminations} \citep{MaronMoore:97,EvenDaral06,Bubeck10BestArm,COLT13,Karnin:al13}.

Fundamental information-theoretic barriers \citep{castro2014,on.the.complexity,GK16} for each specific problem instance refined the worst-case picture, and sparked the development of instance-optimal algorithms for single-answer problems based on
\emph{Track-and-Stop} \citep{GK16} and \emph{Thompson Sampling} \citep{Russo16}. For multiple-answer problems the elegant KL-contraction-based lower bound is not sharp, and new techniques were developed by \cite{eps.del.PAC.bai}.

Recent years also saw a surge of interest in pure exploration with \emph{complex queries} and \emph{structured environments}. \citet{Shivaram:al10} identify the top-$M$ set, \citet{thresholding} the arm closet to a threshold, and \citet{Chen14ComBAI,gabillon16} the optimiser over an arbitrary combinatorial class. For arms arranged in a matrix \citet{rank.one.bandit} study BAI under a rank-one assumption, while \citet{pmlr-v70-zhou17b} seek to identify a Nash equilibrium. For arms arranged in a minimax tree there is significant interest in finding the optimal move at the root \citet{FindTopWinner, maximinarm, structured.bestarm, mcts.by.bai, minimums}, as a theoretical model for studying Monte Carlo Tree search (which is a planning sub-module of many advanced reinforcement learning systems).







\newcommand{\high}{\text{hi}}
\newcommand{\low}{\text{lo}}
\newcommand{\none}{\text{no}}
\newcommand{\onebox}[1]{\tikz[baseline=.25em] \path[fill=#1,semitransparent] (0,0) rectangle (1em,1em);}
\newcommand{\twobox}[2]{\begin{tikzpicture}[baseline=.25em]
  \path[fill=#1,semitransparent] (0,0) rectangle (1em,1em);
  \path[fill=#2,semitransparent] (0,0) rectangle (1em,1em);
\end{tikzpicture}}
\newcommand{\eps}{0.2}
\newcommand{\gam}{0.6}

\begin{table}[t!]
\centering
\renewcommand{\arraystretch}{2}
\begin{tabular}{>{\raggedright\arraybackslash}p{\widthof{Thresholding}}p{\widthof{blablablablabla}}lp{\widthof{example}}p{\widthof{alternative}}}
  Identification Problem
  & Possible answers $\mathcal I$
  & Correct answers $i^*(\vmu) \subseteq \mathcal I$
  & Correct  $i^*(\vmu)$
  & Oracle $i_F(\vmu)$
  \\
  \midrule
  $\epsilon$ Best Arm
  & $[K]$
  & $\setc{k}{\mu_k \ge \max_j \mu_j - \epsilon}$
  & \begin{tikzpicture}[baseline=.5cm]
  \path[fill=red,semitransparent] (0,0) -- (1,0) -- (1,1) -- (1-\eps,1) -- (0,\eps) -- cycle;
  \path[fill=blue,semitransparent] (0,0) -- (0,1) -- (1,1) -- (1,1-\eps) -- (\eps,0) -- cycle;
  \draw  (0,0) rectangle (1,1);
\end{tikzpicture}
  & \begin{tikzpicture}[baseline=.5cm]
  \path[fill=red,semitransparent] (0,0) -- (1,1) -- (1,0) -- cycle;
  \path[fill=blue,semitransparent] (0,0) -- (1,1) -- (0,1) -- cycle;
  \draw  (0,0) rectangle (1,1);
\end{tikzpicture}
  \\
  Thresholding Bandit
  & $2^K$
  & $\set*{\setc*{k}{\mu_k \le \gamma}}$
  & \begin{tikzpicture}[baseline=.5cm]
  \path[fill=brown,semitransparent] (0,0) rectangle (\gam,\gam);
  \path[fill=teal,semitransparent] (\gam,0) rectangle (1,\gam);
  \path[fill=yellow,semitransparent] (0,\gam) rectangle (\gam,1);
  \path[fill=orange,semitransparent] (\gam,\gam) rectangle (1,1);
  \draw  (0,0) rectangle (1,1);
\end{tikzpicture}
  & \begin{tikzpicture}[baseline=.5cm]
  \path[fill=brown,semitransparent] (0,0) rectangle (\gam,\gam);
  \path[fill=teal,semitransparent] (\gam,0) rectangle (1,\gam);
  \path[fill=yellow,semitransparent] (0,\gam) rectangle (\gam,1);
  \path[fill=orange,semitransparent] (\gam,\gam) rectangle (1,1);
  \draw  (0,0) rectangle (1,1);
\end{tikzpicture}
  \\
  $\epsilon$ Minimum Threshold
  & $\set{\low,\high}$
  & $\renewcommand{\arraystretch}{1}
    \begin{array}[t]{@{}ll}
      \set{\low} & \text{if $\min_k \mu_k < \gamma - \epsilon$}
      \\
      \set{\high} & \text{if $\min_k \mu_k > \gamma + \epsilon$}
      \\
      \set{\low, \high} & \text{o.w.}
    \end{array}
    $
  & \begin{tikzpicture}[baseline=.5cm]
  \path[fill=green,semitransparent] (0,0) -- (1,0) -- (1,\gam+\eps) -- (\gam+\eps,\gam+\eps) -- (\gam+\eps,1) -- (0,1) -- cycle;
  \path[fill=orange,semitransparent] (\gam-\eps,\gam-\eps) rectangle (1,1);
  \draw  (0,0) rectangle (1,1);
\end{tikzpicture}
  & \begin{tikzpicture}[baseline=.5cm]
  \newcommand{\descent}{(0.603601, 1.) (0.6061, 0.868323) (0.6086, 0.799213) (0.611099,
  0.754821) (0.613599, 0.723129) (0.616098, 0.698967) (0.618598,
  0.679694) (0.621097, 0.663804) (0.623597, 0.650366) (0.626096,
  0.638769) (0.628595, 0.628595) (0.628595, 0.628595) (0.638769,
  0.626096) (0.650366, 0.623597) (0.663804, 0.621097) (0.679694,
  0.618598) (0.698967, 0.616098) (0.723129, 0.613599) (0.754821,
  0.611099) (0.799213, 0.6086) (0.868323, 0.6061) (1., 0.603601)}
\path[fill=green,semitransparent] (0,0) -- (0,1) -- plot coordinates {\descent}  -- (1,0) -- cycle;
  \path[fill=orange,semitransparent] (1,1) -- (0,1) -- plot coordinates {\descent}  -- (1,0) -- cycle;
  \draw  (0,0) rectangle (1,1);
\end{tikzpicture}
  \\
  Any Low Arm
  & $[K] \cup \set{\none}$
  &
    $\renewcommand{\arraystretch}{1}
    \begin{array}[t]{@{}ll}
      \setc*{k}{\mu_k \le \gamma}
      & \text{if $\min_k \mu_k < \gamma$}
      \\
      \set*{\none} & \text{if $\min_k \mu_k > \gamma$}
    \end{array}$
& \begin{tikzpicture}[baseline=.5cm]
  \path[fill=red,semitransparent] (0,0) rectangle (1,\gam);
  \path[fill=blue,semitransparent] (0,0) rectangle (\gam,1);
  \path[fill=orange,semitransparent] (\gam,\gam) rectangle (1,1);
  \draw  (0,0) rectangle (1,1);
\end{tikzpicture}
& \begin{tikzpicture}[baseline=.5cm]
  \path[fill=red,semitransparent] (0,0) -- (1,0) -- (1,\gam) --  (\gam,\gam) -- cycle;
  \path[fill=blue,semitransparent] (0,0) -- (0,1) -- (\gam,1) -- (\gam,\gam) -- cycle;
  \path[fill=orange,semitransparent] (\gam,\gam) rectangle (1,1);
  \draw  (0,0) rectangle (1,1);
\end{tikzpicture}
  \\
  Any Sign
  & $[K] \times \set{\low, \high}$
  &
    $
    \setc*{
    (k,\low)
    }{\mu_k \le \gamma}
    \cup
    \setc*{
    (k,\high)
    }{\mu_k \ge \gamma}
    $
& \begin{tikzpicture}[baseline=.5cm]
  \path[fill=Melon,semitransparent] (0,0) rectangle (1,\gam);
  \path[fill=OliveGreen,semitransparent] (0,\gam) rectangle (1,1);
  \path[fill=RubineRed,semitransparent] (0,0) rectangle (\gam,1);
  \path[fill=ProcessBlue,semitransparent] (\gam,0) rectangle (1,1);
  \draw  (0,0) rectangle (1,1);
\end{tikzpicture}
  & \begin{tikzpicture}[baseline=.5cm]
    \path[fill=RubineRed,semitransparent] (0,0) -- (1,0) -- (1,2*\gam-1) -- (\gam,\gam) -- cycle;
    \path[fill=OliveGreen,semitransparent] (1,1) -- (1,2*\gam-1) -- (\gam,\gam) -- cycle;
    \path[fill=ProcessBlue,semitransparent] (1,1) -- (\gam,\gam) -- (2*\gam-1,1) -- cycle;
    \path[fill=Melon,semitransparent] (0,0) -- (\gam,\gam) -- (2*\gam-1,1) -- (0,1) -- cycle;
  \draw  (0,0) rectangle (1,1);
\end{tikzpicture}
\end{tabular}
\caption[Identification Problems]{Collection of Identification Problems. The diagrams depict 2-arm instances, parameterised by the two means, with colours showing the set of correct answers: \textbf{one correct answer:}
\onebox{red}~$\set{1}$,
\onebox{blue}~$\set{2}$,
\onebox{yellow}~$\set{\set{1}}$,
\onebox{teal}~$\set{\set{2}}$,
\onebox{brown}~$\set{\set{1,2}}$,
\onebox{green}~$\set{\low}$,
\onebox{orange}~$\set{\emptyset}$/$\set{\high}$/$\set{\none}$,
\onebox{Melon}~$\set{(1,\low)}$,
\onebox{OliveGreen}~$\set{(1,\high)}$,
\onebox{RubineRed}~$\set{(2,\low)}$,
\onebox{ProcessBlue}~$\set{(2,\high)}$,
and \textbf{two correct answers:}
\twobox{red}{blue}~$\set{1,2}$,
\twobox{green}{orange}~$\set{\low,\high}$,
\twobox{Melon}{RubineRed}~$\set{(1,\low), (2,\low)}$,
\twobox{OliveGreen}{RubineRed}~$\set{(1,\low), (2,\high)}$,
\twobox{Melon}{ProcessBlue}~$\set{(1,\high), (2,\low)}$,
\twobox{OliveGreen}{ProcessBlue}~$\set{(1,\high), (2,\high)}$
.
}\label{tab:problems}
\end{table}

\section{Notations}
We work in a given one-parameter one-dimensional canonical exponential family, with mean parameter in an open interval $\mathcal{O} \subseteq \R$. We denote by $d(\mu, \lambda)$ the KL divergence from the distribution with mean $\mu$ to that with mean $\lambda$. A $K$-armed bandit model is identified by its vector $\vmu\in \mathcal O^K$ of mean parameters. We denote by $\mathcal M \subseteq \mathcal{O}^K$ the set of possible mean parameters in a specific bandit problem. Excluding parts of $\mathcal{O}^K$ from $\mathcal{M}$ may be used to encode a known structure of the problem. We assume that there is a finite domain $\mathcal I$ of answers, and that the \emph{correct answer} for each bandit model is specified by a set-valued function $i^*: \mathcal M \to 2^{\mathcal I}$. We include visual examples in Table~\ref{tab:problems}.

\paragraph{Algorithms.}
A learning strategy is parametrised by a stopping rule $\tau_\delta \in \N$ depending on a parameter $\delta\in[0,1]$, a sampling rule $A_1, A_2, \ldots \in [K]$, and a recommendation rule $\ihat \in \mathcal I$. When a learning strategy meets a bandit model $\vmu$, they interactively generate a history $A_1, X_1, \ldots, A_\tau, X_\tau, \ihat$, where $X_t \sim \mu_{A_t}$. We allow the possibility of non-termination $\tau_\delta = \infty$, in which case there is no recommendation $\ihat$. At stage $t\in\N$, we denote by $N_t=(N_{t,1}, \ldots, N_{t,K})$ the number of samples (or ``pulls'') of the arms, and by $\hat{\vmu}_t$ the vector of empirical means of the samples of each arm.

\paragraph{Confidence and sample complexity.}
For confidence parameter $\delta \in (0,1)$, we say that a strategy is $\delta$-correct (or $\delta$-PAC) for bandit model $\vmu$ if it recommends a correct answer with high probability, i.e.\ $\pr_\vmu\del[\big]{\text{$\tau_\delta < \infty$ and $\ihat \in i^*(\vmu)$}} \ge 1-\delta$. We call a given strategy $\delta$-correct if it is $\delta$-correct for every $\vmu \in \mathcal M$. We measure the statistical efficiency of a strategy on a bandit model $\vmu$ by its \emph{sample complexity} $\ex_\vmu[\tau_\delta]$. We are interested in $\delta$-correct algorithms minimizing sample complexity.

\paragraph{Divergences.}
For any answer $i \in \mathcal I$, we define the \emph{alternative to $i$}, denoted $\neg i$, to be the set of bandit models on which answer $i$ is incorrect, i.e.\
\[
  \neg i
  ~\df~
  \setc*{\vmu \in \mathcal M}{i \notin i^*(\vmu)} \: .
\]

We define the ($\w$-weighted) divergence from $\vmu \in \mathcal M$ to $\vlambda \in \mathcal M$ or $\Lambda \subseteq \mathcal M$ by
\begin{align*}
  D(\w, \vmu, \vlambda) &~=~ \sum_k w_k d(\mu_k, \lambda_k)
  &
  D(\w, \vmu, \Lambda) &~=~ \inf_{\vlambda \in \Lambda} D(\w, \vmu, \vlambda)
  \\
  D(\vmu, \Lambda) &~=~ \sup_{\w \in \triangle_K} D(\w, \vmu, \Lambda)
  &
  D(\vmu) &~=~ \max_{i \in \mathcal I} D(\vmu, \neg i)
\end{align*}
Note that $D(\w, \vmu, \Lambda) = 0$ whenever $\vmu \in \Lambda$. We denote by $i_F(\vmu)$ the $\argmax$ (set of maximisers) of $i \mapsto D(\vmu, \neg i)$, and we call $i_F(\vmu)$ the \emph{oracle answer(s)} at $\vmu$. We write $\w^*(\vmu, \neg i)$ for the maximisers of $\w\mapsto D(\w, \vmu, \neg i)$, and call these the \emph{oracle weights for answer $i$} at $\vmu$. We write  $\w^*(\vmu) = \bigcup_{i \in i_F(\vmu)} \w^*(\vmu, \neg i)$ for the set of \emph{oracle weights} among all oracle answers. We include expressions for the divergence when $i^*$ is generated by half-spaces, minima and spheres in Appendix~\ref{app:divergences}.

The function $i_F(\vmu) = \{i\in\mathcal{I}: D(\vmu,\neg i) = D(\vmu)\}$ is set valued, as is $\w^*$. They are singletons with continuous value over some connected subsets of $\mathcal{M}$, and are multi-valued on common boudaries of two or more such sets. $i_F$ and $\w^*$ can be thought of as having single values, unless $\vmu$ sits on such a boundary, in which case we will prove that they are equal to the union (or convex hull of the union) of their values in the neighbouring regions.

\section{Lower Bound}

We show a lower bound for any algorithm for multiple-answer problems. Our lower bound extends the single-answer result of \cite{GK16}. We are further inspired by \cite{eps.del.PAC.bai}, who analyse the $\epsilon$-BAI problem. They prove lower bounds for algorithms with a sampling rule independent of $\delta$, imposing the further restriction that either $K=2$ or that the algorithm ensures that $N_{t,k}/t$ converges as $t \to \infty$. The new ingredient in this section is a game-theoretic equilibrium argument, which allows us to analyse any $\delta$-correct algorithm in any multiple answer problem. Our main lower bound is the following.

\begin{theorem}\label{thm:lbd}
Any $\delta$-correct algorithm verifies
\begin{align*}
  \liminf_{\delta\to 0}\frac{\ex_\vmu[\tau_\delta]}{\log(1/\delta)} ~\geq~
  T^*(\vmu)
  ~\df~
  D(\vmu)^{-1}
  \quad
  \text{where}
  \quad
  D(\vmu)
  ~=~
  \max_{i\in i^*(\vmu)} \max_{\w \in \triangle_K} \inf_{\vlambda \in \neg i} \sum_{k=1}^K w_k d(\mu_k,\lambda_k)
\end{align*}
for any multiple answer instance $\vmu$ with  sub-Gaussian arm distributions.
\end{theorem}
The proof is in Appendix~\ref{pf:lbd}, where we also discuss how the convenient sub-Gaussian assumption can be relaxed. We would like to point out one salient feature here. To show sample complexity lower bounds at $\vmu$, one needs to find problems that are hard to distinguish from it statistically, yet require a different answer. We obtain these problems by means of a minimax result.
\begin{lemma}\label{lem:nash_finitely_supported}\label{lem:equilibrium}
For any answer $i \in \mathcal I$, the divergence from $\vmu$ to $\neg i$ equals
\[
  D(\vmu, \neg i)
  ~=~
  \inf_{\pr} \max_{k \in [K]}
  \ex_{\vlambda \sim \pr} \sbr*{
    d(\mu_k, \lambda_k)
  }
  .
\]
where the infimum ranges over probability distributions on $\neg i$ supported on (at most) $K$ points.
\end{lemma}
The proof of Theorem~\ref{thm:lbd} then challenges any algorithm for $\vmu$ by obtaining a witness $\pr$ for $D(\vmu) = \max_i D(\vmu, \neg i)$ from Lemma~\ref{lem:nash_finitely_supported}, sampling a model $\vlambda \sim \pr$, and showing that if the algorithm stops early, it must make a mistake w.h.p.\ on at least one model from the support. The equilibrium property of $\pr$ allows us to control a certain likelihood ratio martingale regardless of the sampling strategy employed by the algorithm.

We discuss the novel aspect of Theorem~\ref{thm:lbd} and its lessons for the design of optimal algorithms. First of all, for single-answer instances $\card{i^*(\vmu)}{=}1$ we recover the known asymptotic lower bound \cite[Remark~2]{GK16}. For multiple-answer instances the bound says the following. The optimal sample complexity hinges on the \emph{oracle answers} $i_F(\vmu)$. That is, for $i_f \in i_F(\vmu)$, the complexity of problem $\vmu$ is governed by the difficulty of discriminating $\vmu$ from the set of models on which answer $i_f$ is incorrect.

Is the bound tight? We argue yes. Consider the following oracle strategy, which is specifically designed to be very good at $\vmu$. First, it computes a pair $(i, \w)$ witnessing the two outer maxima in $D(\vmu)$. The algorithm samples according to $\w$. It stops when it can statistically discriminate $\hat\vmu_t$ from $\neg i$, and outputs $\ihat = i$. This algorithm will indeed have expected sample complexity equal to $D(\vmu)^{-1}$ at $\vmu$, and it will be correct.

The above oracle viewpoint presents an idea for designing algorithms, following \cite{GK16}. Perform a lower-order amount of forced exploration of all arms to ensure $\hat\vmu_t \to \vmu$. Then at each time point compute the empirical mean vector $\hat{\vmu}_t$ and oracle weights $\w_t \in \w^*(\hat\vmu_t)$. Then sample according to $\w_t$. This approach is successful for single-answer bandits with unique and continuous oracle weights. We argue in Section~\ref{sec:tas} below that it extends to points of discontinuity by exploiting upper hemicontinuity and convexity of $\w^*$.

For multiple-answer bandits, we argue that the set of maximisers $\w^*(\vmu)$ is no longer convex when $i_F(\vmu)$ is not a singleton. It can then happen that $\hat\vmu_t \to \vmu$, while at the same time $\w^*(\hat\vmu_t)$ keeps oscillating. If the algorithm tracks $\w^*(\hat\vmu_t)$, its sampling proportions will end up in the convex hull of $\w^*(\vmu)$. However, as $\w^*(\vmu)$ is not convex itself, these proportions will not be optimal. We present empirical evidence for that effect in section~\ref{sec:failure_tas}. The lesson here is that the oracle needs to pick an answer and ``stick with it''. This will be the basis of our algorithm design in Section~\ref{sec:alg}.


\section{Properties of the Optimal Allocation Sets}

The Track-and-Stop sampling strategy aims at ensuring that the sampling proportions converge to oracle weights. In the case of a singleton-valued oracle weights set $\w^*(\vmu)$ for single answer problems, that convergence was proven in \citep{GK16}. We study properties of that set with the double aim of extending Track-and-Stop to points $\vmu$ where $\w^*(\vmu)$ is not a singleton and of highlighting what properties hold only for the single-answer case, but not in general.

\subsection{Continuity}

We first prove continuity properties of $D$ and $\w^*$. We show how the convergence of $\hat{\vmu}_t$ to $\vmu$ translates into properties of the divergences from $\hat{\vmu}_t$ to the alternative sets.

For a set $B$, let $\mathbb{S}(B)=2^B\setminus \{\emptyset\}$ be the set of all \textit{non-empty} subsets of $B$.

\begin{definition}[Upper hemicontinuity]
A set-valued function $\Gamma:A \to \mathbb{S}(B)$ is upper hemicontinuous at $a\in A$ if for any open neighbourhood $V$ of $\Gamma(a)$ there exists a neighbourhood $U$ of $a$ such that for all $x\in U$, $\Gamma(x)$ is a subset of $V$. 
\end{definition}

\begin{theorem}\label{th:continuity}
For all $i\in\mathcal{I}$,
\begin{enumerate}
\item the function $(\w,\vmu)\mapsto D(\w, \vmu, \neg i)$ is continuous on $\triangle_K\times\mathcal{M}$,
\item $\vmu\mapsto D(\vmu, \neg i)$ and $\vmu\mapsto D(\vmu)$ are continuous on $\mathcal{M}$,
\item $\vmu\mapsto \w^*(\vmu, \neg i)$, $\vmu\mapsto \w^*(\vmu)$ and $\vmu\mapsto i_F(\vmu)$ are upper hemicontinuous on $\mathcal{M}$ with non-empty and compact values,
\end{enumerate}
\end{theorem}
Proof in Appendix~\ref{sec:continuity_proof}. It uses Berge's maximum theorem and a modification thereof due to \citep{FKV14}. Related continuity results using this type of arguments, but restricted to single-valued functions, appeared for the regret minimization problem in \citep{combes2017minimal}.

\subsection{Convexity}

\begin{proposition}
For each $i\in\mathcal{I}$, for all $\vmu\in\mathcal{M}$ the set $\w^*(\vmu, \neg i)$ is convex.

If $i_F(\vmu)$ is a singleton, then $\w^*(\vmu) = \cup_{i\in i_F(\vmu)} \w^*(\vmu, \neg i)$ is convex.
\end{proposition}

This is a consequence of the concavity of $\w\mapsto D(\w,\vmu,\neg i)$ (which is an infimum of linear functions). In single-answer problems, we obtain that the oracle weights set $\w^*(\vmu)$ is convex everywhere. This is however not the case in general for multiple-answer problems, as illustrated by the next example, which is called \emph{Any Low Arm} in Table~\ref{tab:problems}.

Consider two arms with Gaussian distributions $(\mathcal{N}(\mu_k, 1))_{k=1,2}$ for $\vmu \in \R^2$. The query is: ``is there an arm $k$ with $\mu_k <0$ ? if yes, return one''. The possible answers are $\{\none, 1, 2\}$. For $\vmu$ with two negative coordinates both answers 1 and 2 are correct. In that case, the divergence from $\vmu$ to the alternatives is $D(\vmu) = \sup_{\w\in\triangle_2}\max_{k=1,2} w_k d(\mu_k, 0) = \max_{k=1,2}d(\mu_k, 0)$.

For $\mu_1 < \mu_2 < 0$, $\w^*(\vmu) = \{(1,0)\}$.
For $\mu_2 < \mu_1 < 0$, $\w^*(\vmu) = \{(0,1)\}$.

For $\mu_1=\mu_2<0$, $\w^*(\vmu) = \{(1,0), (0,1)\}$, which is not convex.

That example also illustrates the upper hemicontinuity of $\w^*(\vmu)$: since $\vmu$ of the form $(\mu,\mu)$ is the limit of a sequence $(\vmu_t)_{t\in\N}$ with $\mu_{t,1} {<} \mu_{t,2}$, we obtain that $\{(1,0)\} \subseteq \w^*(\vmu)$. Similarly, using a sequence with $\mu_{t,1} {>} \mu_{t,2}$, $\{(0,1)\} \subseteq \w^*(\vmu)$. Playing intermediate weights $\w = (\alpha,1-\alpha)$ results in strictly sub-optimal $D(\vmu, \w) = \max\set*{\alpha,1-\alpha} d(\mu, 0) < d(\mu,0) = D(\vmu)$.

\subsection{Consequences for Track-and-Stop}\label{sec:tas}

The original analysis of Track-and-Stop excludes the mean vectors $\vmu\in\mathcal{M}$ for which $\w^*(\vmu)$ is not a singleton. We show that the upper hemicontinuity and convexity properties of $\w^*(\vmu)$ allow us to extend that analysis to all $\vmu$ with a single oracle answer (in particular all single-answer bandit problems), at least for one of the two Track-and-Stop variants. Indeed, that algorithm was introduced with two possible subroutines, dubbed C-tracking and D-tracking \citep{GK16}. Both variants compute oracle weights $\w_t$ at the point $\hat{\vmu}_t$, but the arm pulled differs.

\textit{C-tracking}: compute the projection $\w^{\varepsilon_t}_t$ of $\w_t$ on $\triangle^{\varepsilon_t}_K = \{\w\in\triangle_K: \forall k\in[K], w_k \geq \varepsilon_t\}$, where $\varepsilon_t>0$. Pull the arm with index $k_t = \arg\min_{k\in[K]} N_{t,k} - \sum_{s=1}^t w_{s,k}^{\varepsilon_s}$.

\textit{D-tracking}: if there is an arm $j$ with $N_{t,j} \leq \sqrt{t}-K/2$, then pull $k_t = j$. Otherwise, pull the arm $k_t = \arg\min_{k\in[K]} N_{t,k} - t w_{t,k}$ .

The proof of the optimal sample complexity of Track-and-Stop for C-tracking remains essentially unchanged but we replace Proposition 9 of \citep{GK16} by the following lemma, proved in Appendix~\ref{sec:proof_lem_avegage_w_converges}.

\begin{lemma}\label{lem:average_w_converges_to_optimal_if_convex}
Let a sequence $(\hat{\vmu}_t)_{t\in\N}$ verify $\lim_{t\to+\infty}\hat{\vmu}_t = \vmu$ . For all $t\geq 0$, let $\w_t\in \w^*(\hat{\vmu}_t)$ be arbitrary oracle weights for $\hat{\vmu}_t$ . If $\w^*(\vmu)$ is convex, then 
\begin{align*}
\lim_{t\to +\infty}\inf_{\w \in \w^*(\vmu)} \left\Vert\frac{1}{t}\sum_{s=1}^t \w_s - \w\right\Vert_\infty = 0 \: .
\end{align*}
\end{lemma}

The average of oracle weights for $\hat{\vmu}_t$ converges to the set of oracle weights for $\vmu$. C-tracking then ensures that the proportion of pulls $N_t/t$ is close to that average by Lemma 7 of \citep{GK16}, hence $N_t/t$ gets close to oracle weights.

\begin{theorem}\label{th:tas_sample_complexity}
For all $\vmu\in\mathcal{M}$ such that $i_F(\vmu)$ is a singleton (in particular all single-answer problems), Track-and-Stop with C-tracking is $\delta$-correct with asymptotically optimal sample complexity.
\end{theorem}

Proof in Appendix \ref{sec:proof_TaS_sample_complexity}. We encourage the reader to first proceed to section~\ref{sec:alg}, since the proof considers the result as a special case of the multiple-answers setting.

\begin{remark}\label{rem:D_tracking}
If $\w^*(\vmu)$ is not a singleton, Track-and-Stop using D-tracking may fail to converge to $\w^*(\vmu)$, even when it is convex.
\end{remark}

While we do not prove that D-tracking fails to converge to $\w^*(\vmu)$ on a specific example of a bandit, we provide empirical evidence in Appendix~\ref{sec:failure_D_tracking}. The reason for the failure of D-tracking is that it does not in general converge to the convex hull of the points it tracks. Suppose that $\w_t = \w^{(1)} = (1/2, 1/2, 0)$ for $t$ odd and $\w_t = \w^{(2)} = (1/2, 0, 1/2)$ for $t$ even. Then D-tracking verifies $\lim_{t\to +\infty} N_t/t = (1/3, 1/3, 1/3)$. This limit is outside of the convex hull of $\{\w^{(1)}, \w^{(2)}\}$.


\section{Algorithms for the Multiple-Answers Setting}\label{sec:alg}
\tikzset{
  pt/.style={fill,circle,inner sep=.25ex},
}
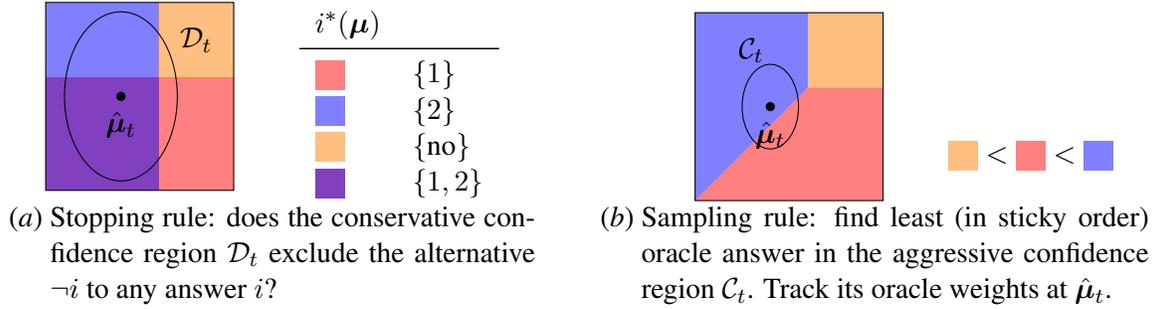
\begin{figure}[t]
  \centering
  \subfigure[Stopping rule: does the conservative confidence region $\mathcal D_t$ exclude the alternative $\neg i$ to any answer $i$?]{
    \quad
    \begin{tikzpicture}[baseline=0,scale=2.5]
      \path[fill=red,semitransparent] (0,0) rectangle (1,\gam);
      \path[fill=blue,semitransparent] (0,0) rectangle (\gam,1);
      \path[fill=orange,semitransparent] (\gam,\gam) rectangle (1,1);
      \draw  (0,0) rectangle (1,1);
      \coordinate (hmu) at (.4,.5);
      \node[pt] at (hmu) [label=270:{$\hat\vmu_t$}] {};
      \draw (hmu) ellipse (.3 and .45);
      \node at (.8,.8) {$\mathcal D_t$};
    \end{tikzpicture}
    \qquad
    \begin{tabular}[b]{ll}
      $i^*(\vmu)$
      \\
      \midrule
      \onebox{red} &$\set{1}$
      \\
      \onebox{blue} &$\set{2}$
      \\
      \onebox{orange} &$\set{\none}$
      \\
      \twobox{red}{blue} &$\set{1,2}$
    \end{tabular}
    \quad
  }
  \qquad
  \subfigure[Sampling rule: find least (in sticky order) oracle answer in the aggressive confidence region $\mathcal C_t$. Track its oracle weights at $\hat \vmu_t$.
  ]{
    \qquad\quad
    \begin{tikzpicture}[baseline=.5cm,scale=2.5]
      \path[fill=red,semitransparent] (0,0) -- (1,0) -- (1,\gam) --  (\gam,\gam) -- cycle;
      \path[fill=blue,semitransparent] (0,0) -- (0,1) -- (\gam,1) -- (\gam,\gam) -- cycle;
      \path[fill=orange,semitransparent] (\gam,\gam) rectangle (1,1);
      \draw  (0,0) rectangle (1,1);
      \coordinate (hmu) at (.4,.5);
      \node[pt] at (hmu) [label=270:{$\hat\vmu_t$}] {};
      \draw (hmu) ellipse (.15 and .225);
      \node at (.3,.8) {$\mathcal C_t$};
    \end{tikzpicture}
    \qquad
    $\text{\onebox{orange}} <
    \text{\onebox{red}} <
    \text{\onebox{blue}}$
    \quad
  }
  \\
  \caption{Sticky Track-and-Stop: The two main ideas, illustrated on the \textit{Any Low Arm} problem.
  }
\end{figure}

We can prove for Track-and-Stop the following suboptimal upper bound on the sample complexity, based on the fact that it ensures convergence of $N_t/t$ to the convex hull of the oracle weight set.

\begin{theorem}\label{th:tas_sample_complexity_general}
Let $\conv(A)$ be the convex hull of a set $A$. For all $\vmu\in\mathcal{M}$ in a multi-answer problem, Track-and-Stop with C-tracking is $\delta$-correct and verifies
\begin{align*}
\lim_{\delta \to 0} \frac{\ex_\vmu[\tau_\delta]}{\log(1/\delta)} \leq \max_{\w\in \conv(\w^*(\vmu))} \frac{1}{D(\w, \vmu)} \: .
\end{align*}
\end{theorem}

\subsection{Sticky Track-and-Stop}

\begin{algorithm}
\caption{Sticky Track-and-Stop.}
\label{alg:sticky}
	Input: $\delta>0$, strict total order on $\mathcal{I}$. Set $t = 1$ , $\hat{\vmu}_0 = 0$ , $N_0 = 0$  .\\
	\While{not stopped}{
  	  Let $\mathcal{C}_t = \{\vmu'\in\mathcal{M}\: : \: D(N_{t-1}, \hat{\vmu}_{t-1}, \vmu') \leq \log(f(t-1))\}$ .
          \tcp*[f]{small conf.\ reg.}
          \\
  	  Compute $I_t = \bigcup_{\vmu' \in \mathcal{C}_t} i_F(\vmu')$ .\\
      Pick the first alternative $i_t\in I_t$ in the order on $\mathcal{I}$.\\
      Compute $\w_t \in \w^*(\hat{\vmu}_{t-1}, \neg i_t)$.\\
      Pull an arm $a_t$ according to the C-tracking rule and receive $X_t \sim \nu_{a_t}$ .\\
      Set $N_t = N_{t-1} + \e_{a_t}$  and 
      $\hat{\vmu}_t = \hat{\vmu}_{t-1} + \frac{1}{N_{t,a_t}}(X_t-\hat{\mu}_{t-1,a_t})\e_{a_t}$ . \\
      Let $\mathcal{D}_t = \{\vmu'\in\mathcal{M}\: : \: D(N_t, \hat{\vmu}_t, \vmu') \leq \beta(t, \delta)\}$ .
      \tcp*[f]{large conf.\ reg.}
      \\
	  \If{there exists $i\in\mathcal{I}$ such that $\mathcal{D}_t \cap \neg i = \emptyset$}
	  	{stop and return $i$.}
	  $t \leftarrow t+1$ .
	}
\end{algorithm}

The cases of multiple-answers problems for which Track-and-Stop is inadequate are $\vmu\in\mathcal{M}$ with cardinal of $i_F(\vmu)$ greater than 1. When convexity does not hold, $\w^*(\vmu)$ is the union of the convex sets $(\w^*(\vmu, \neg i))_{i\in i_F(\vmu)}$. If an algorithm can a priori select $i_f\in i_F(\vmu)$ and track allocations $\w_t$ in $\w^*(\hat{\vmu}_t, \neg i_f)$, then using Track-and-Stop on that restricted problem will ensure that $N_t/t$ converge to oracle weights. Our proposed algorithm, Sticky Track-and-Stop, which we display in Algorithm~\ref{alg:sticky}, uses a confidence region around the current estimate $\hat{\vmu}_t$ to determine what $i\in\mathcal{I}$ can be the oracle answer for $\vmu$. It selects one of these answers according to an arbitrary total order on $\mathcal{I}$ and does not change it (sticks to it) until no point in the confidence region has it in its set of oracle answers.

\begin{theorem}
For $\beta(t, \delta) = \log(C t^2/\delta)$, with $C$ such that $C \geq e \sum_{t=1}^{+\infty} (\frac{e}{K})^K \frac{(\log^2(Ct^2)\log(t))^K}{t^2}$, Sticky Track-and-Stop is $\delta$-correct.
\end{theorem}

That result is a consequence of Proposition 12 of \citep{GK16}.

\subsection{Sample Complexity}

\begin{theorem}\label{th:STaS_optimal_sample_complexity}
Sticky Track-and-Stop is asymptotically optimal, i.e. it verifies for all $\vmu\in\mathcal{M}$,
\begin{align*}
\lim_{\delta \to 0} \frac{\ex_\vmu[\tau_\delta]}{\log(1/\delta)} \to \frac{1}{D(\vmu)} \: .
\end{align*}
\end{theorem}

Let $i_\vmu= \min i_F(\vmu)$ in the arbitrary order on answers. For $\varepsilon,\xi>0$, we define $C_{\varepsilon,\xi}^*(\vmu)$, the minimal value of $D(\w',\vmu',\neg i_\vmu)$ for $\w'$ and $\vmu'$ in $\varepsilon$ and $\xi$-neighbourhoods of $\w^*(\vmu)$ and $\vmu$.
\begin{align*}
C_{\varepsilon,\xi}^*(\vmu) =
\inf_{
	\substack{\vmu':\Vert\vmu' - \vmu\Vert_\infty \leq \xi\\
	 \w': \inf_{\w\in\w^*(\vmu, \neg i_\vmu)}\Vert \w' - \w \Vert_\infty \leq 3\varepsilon}
}
D(\w',\vmu',\neg i_\vmu) \: . 
\end{align*}

Our proof strategy is to show that under a concentration event defined below, for $t$ big enough, $(\hat{\vmu}_t, N_t/t)$ belongs to that $(\xi, \varepsilon)$ neighbourhood of $(\vmu, \w^*(\vmu, \neg i_\vmu))$. From that fact, we obtain $D(N_t, \hat{\vmu}_t, \neg i_\vmu) \geq t C_{\varepsilon,\xi}^*(\vmu)$. Furthermore, if the algorithm does not stop at stage $t$, we also get an upper bound on $D(N_t, \hat{\vmu}_t, \neg i_\vmu)$ from the stopping condition. We obtain an upper bound on the stopping time, function of $\delta$ and $C_{\varepsilon,\xi}^*(\vmu)$. By continuity of $(\w,\vmu)\mapsto D(\w, \vmu, \neg i_\vmu)$ (from Theorem~\ref{th:continuity}), we have $\lim_{\varepsilon\to 0, \xi \to 0} C_{\varepsilon,\xi}^*(\vmu) = D(\vmu, \neg i_\vmu) = D(\vmu)$.

\paragraph{Two concentration events.} 

Let $\mathcal{E}_T^{}=\bigcap_{t=h(T)}^T\{\vmu \in \mathcal{C}_t\}$ be the event that the small confidence region contains the true parameter vector $\vmu$ for $t\geq h(T)$. The function $h:\N\to\R$, positive, increasing and going to $+\infty$, makes sure that each event $\{\vmu \in \mathcal{C}_t\}$ appears in finitely many $\mathcal{E}_T$, which will be essential in the concentration results. We will eventually use $h(T) = \sqrt{T}$.

In order to define the second event, we first highlight a consequence of Theorem~\ref{th:continuity}.

\begin{corollary}\label{lem:w*_continuity}
For all $\varepsilon>0$, for all $\vmu\in\mathcal{M}$, for all $i\in\mathcal{I}$, there exists $\xi>0$ such that
\begin{align*}
\Vert \vmu' - \vmu\Vert_\infty \leq \xi \Rightarrow \forall \w' \in \w^*(\vmu', \neg i)\: \exists \w \in \w^*(\vmu, \neg i), \: \Vert \w' - \w \Vert_\infty \leq \varepsilon \: .
\end{align*}
\end{corollary}

Let $\mathcal{E}_T' = \bigcap_{t=h(T)}^T\{\Vert \hat{\vmu}_t - \vmu\Vert_\infty \leq \xi\}$ be the event that the empirical parameter vector is close to $\vmu$, where $\xi$ is chosen as in the previous corollary for $i=i_\vmu$. The analysis of Sticky Track-and-Stop consists of two parts: first show that $\mathcal{E}_T^c$ and ${\mathcal{E}_T'}^c$ happen rarely enough to lead only to a finite term in $\ex_\vmu[\tau_\delta]$; then show than under $\mathcal{E}_T^{} \cap \mathcal{E}_T'$ there is an upper bound on $\tau_\delta$.

\begin{lemma}\label{lem:tau_delta_decomposition}
Suppose that there exists $T_0$ such that for $T\geq T_0$, $\mathcal{E}_T^{} \cap \mathcal{E}_T' \subset \{\tau_\delta \leq T\}$. Then
\begin{align}
\ex_\vmu[\tau_\delta]
\leq T_0 + \sum_{T=T_0}^{+\infty} \pr_\vmu(\mathcal{E}_T^c) + \sum_{T=T_0}^{+\infty} \pr_\vmu({\mathcal{E}_T'}^c) \: .\label{eq:tau_delta_decomposition}
\end{align}
\end{lemma}
\begin{proof}
Since $\tau_\delta$ is a non-negative integer-valued random variable, $\ex_\vmu[\tau_\delta] = \sum_{t=0}^{+\infty} \pr_\vmu\{\tau_\delta>t\}$. For $t\geq T_0$, $\pr_\vmu\{\tau_\delta > t\} \leq \pr_\vmu(\mathcal{E}_T^{c} \cup {\mathcal{E}_T'}^c) \leq \pr_\vmu(\mathcal{E}_T^{c}) + \pr_\vmu( {\mathcal{E}_T'}^c)$.
\end{proof}

The sums depending on the events $\mathcal{E}_T^{}$ and $\mathcal{E}_T'$ in~\eqref{eq:tau_delta_decomposition} are finite for well chosen $h(T)$ and $\mathcal{C}(t)$.

\begin{lemma}\label{lem:when_concentration_fails}
For $h(T) = \sqrt{T}$ and $f(t) = \exp(\beta(t, 1/t^5))=Ct^{10}$ in the definition of the confidence region $\mathcal{C}_t$, the sum $\sum_{T=T_0}^{+\infty} \pr_\vmu(\mathcal{E}_T^c) + \sum_{T=T_0}^{+\infty} \pr_\vmu({\mathcal{E}_T'}^c)$ is finite.
\end{lemma}

Proof in appendix~\ref{sec:proof_lem_when_concentration_fails}. The remainder of the proof is concerned with finding a suitable $T_0$. First, we show that if $\hat{\vmu}_t$ and $N_t$ are in an $(\xi, \varepsilon)$ neighbourhood of $\vmu$ and $\w^*(\vmu, \neg i_\vmu)$, then such an upper bound $T_0$ on $\tau_\delta$ can be obtained. The next lemma is proved in Appendix~\ref{sec:proof_lem_optimal_pulling_implies_small_stopping_time}.

\begin{lemma}\label{lem:optimal_pulling_implies_small_stopping_time}
Suppose that there exists a time $T_1$ and a function $\eta(t) < t-1$ such that if $\mathcal{E}_T\cap\mathcal{E}_T'$ holds then for $t\geq \eta(T)$, $D(N_t,\hat{\vmu}_t,\neg i_\vmu) \geq t C^*_{\varepsilon, \xi}(\vmu)$. Then when that event holds,
\begin{align*}
\tau_\delta \leq T_0 = \max(T_1, \inf\{T\in\N: 1 + \frac{\beta(T,\delta)}{C^*_{\varepsilon,\xi}(\vmu)} \leq T\}) \: .
\end{align*}
\end{lemma}

\paragraph{The oracle answer $i_t$ becomes constant.} Due to the forced exploration present in the C-tracking procedure, the confidence region $\mathcal{C}_t$ shrinks. After some time, when concentration holds, the set of possible oracle answers $I_t$ becomes constant over $t$ and equal to $i_F(\vmu)$.

\begin{lemma}\label{lem:convergence_oracle_answer}
If an algorithm guaranties that for all $k\in[K]$ and all $t\geq 1$, $N_{t,k} \geq n(t)>0$ with $\lim_{t\to +\infty}n(t)/\log(f(t)) = +\infty$, then there exists $T_\Delta$ such that under the event $\mathcal{E}_T$, for $t\geq \max(h(T), T_\Delta)$, $I_t = i_F(\vmu)$ and $\min I_t = i_\vmu = \min i_F(\vmu)$.
\end{lemma}
Proof in Appendix~\ref{sec:proof_lem_convergence_oracle_answer}. Note that Lemma~\ref{lem:convergence_oracle_answer} depends only on the amount of forced exploration and not on other details of the algorithm. Any algorithm using C-tracking verifies the hypothesis with $n(t) = \sqrt{t + K^2} - 2K$ by Lemma~\ref{lem:tracking} \citep[Lemma 7]{GK16}.

\paragraph{Convergence to the neighbourhood of $(\vmu, \w^*(\vmu, \neg i_\vmu))$.} Once $i_t=i_\vmu$, we fall back to tracking points from a convex set of oracle weights. The estimate $\hat{\vmu}_t$ and $N_t/t$ both converge, to $\vmu$ and to the set $\w^*(\vmu, \neg i_\vmu)$.
The Lemma below is proved in Appendix~\ref{sec:proof_lem_N_t_converges_to_w*}.

\begin{lemma}\label{lem:N_t_converges_to_w*}
Let $T_\Delta$ be defined as in Lemma~\ref{lem:convergence_oracle_answer}. For $T$ such that $h(T)\geq T_\Delta$, it holds that on $\mathcal{E}_T^{} \cap \mathcal{E}_T'$ Sticky Track-and-Stop with C-Tracking verifies
\begin{align*}
\forall t\geq h(T), \: \Vert \hat{\vmu}_t - \vmu\Vert_\infty \leq \xi \: ,
\quad
\mbox{and}
\quad
\forall t\geq 4\frac{K^2}{\varepsilon^2} + 3\frac{h(T)}{\varepsilon}, \: \inf_{\w\in \w^*(\vmu, \neg i_\vmu)} \Vert \frac{\bm{N}_t}{t} - \w\Vert_\infty \leq 3\varepsilon \: .
\end{align*}
\end{lemma}

\paragraph{Remainder of the proof.}

Let $T_\Delta$ be defined as in Lemma~\ref{lem:convergence_oracle_answer}. Let $T$ be such that $h(T)\geq T_\Delta$. Let $\eta(T) = 4\frac{K^2}{\varepsilon^2} + 3\frac{h(T)}{\varepsilon}$ .
By Lemma \ref{lem:N_t_converges_to_w*}, for $t\geq \eta(T)$, if $\mathcal{E}_T^{}\cap \mathcal{E}_T'$ then $D(N_t, \hat{\vmu}_t, \neg i_\vmu) \geq t C^*_{\varepsilon, \xi}(\vmu)$.

We now apply Lemma~\ref{lem:optimal_pulling_implies_small_stopping_time}. $\eta(T)< T-1$ if $h(T) < \frac{\varepsilon}{3}(T-1) - \frac{4}{3}\frac{K^2}{\varepsilon}$. For $h(T) = \sqrt{T}$ and $T$ bigger than a constant $T_\eta$ depending on $K$ and $\varepsilon$, this is true. Then under $\mathcal{E}_T^{}\cap \mathcal{E}_T'$, the hypotheses of Lemma~\ref{lem:optimal_pulling_implies_small_stopping_time} are verified with $T_1 = h^{-1}(\max(T_\Delta, T_\eta))$.
 
The hypotheses of Lemma~\ref{lem:tau_delta_decomposition} are verified for
$
T_0
= \max(T_1, \inf\{T: 1 + \frac{\beta(T,\delta)}{C^*_{\varepsilon, \xi}(\vmu)} \leq T\})
$ .

Note that $\lim_{\delta \to 0} \frac{T_0}{\log(1/\delta)} = \frac{1}{C_{\varepsilon, \xi}^*(\vmu)}$. Taking $\varepsilon\to 0$ (hence $\xi\to 0$ as well), we obtain $\lim_{\delta\to 0}\frac{\ex_\vmu[\tau_\delta]}{\log(1/\delta)} = \frac{1}{\lim_{\varepsilon\to 0}C_{\varepsilon,\xi}^*(\vmu)} = \frac{1}{D(\vmu)}$ . We proved Theorem~\ref{th:STaS_optimal_sample_complexity}.

\section{Vanilla Track and Stop Fails for Multiple Answers}\label{sec:failure_tas}

\begin{figure}
  \subfigure[Histogram of stopping time $\tau$]{%
    \includegraphics[width=.5\textwidth]{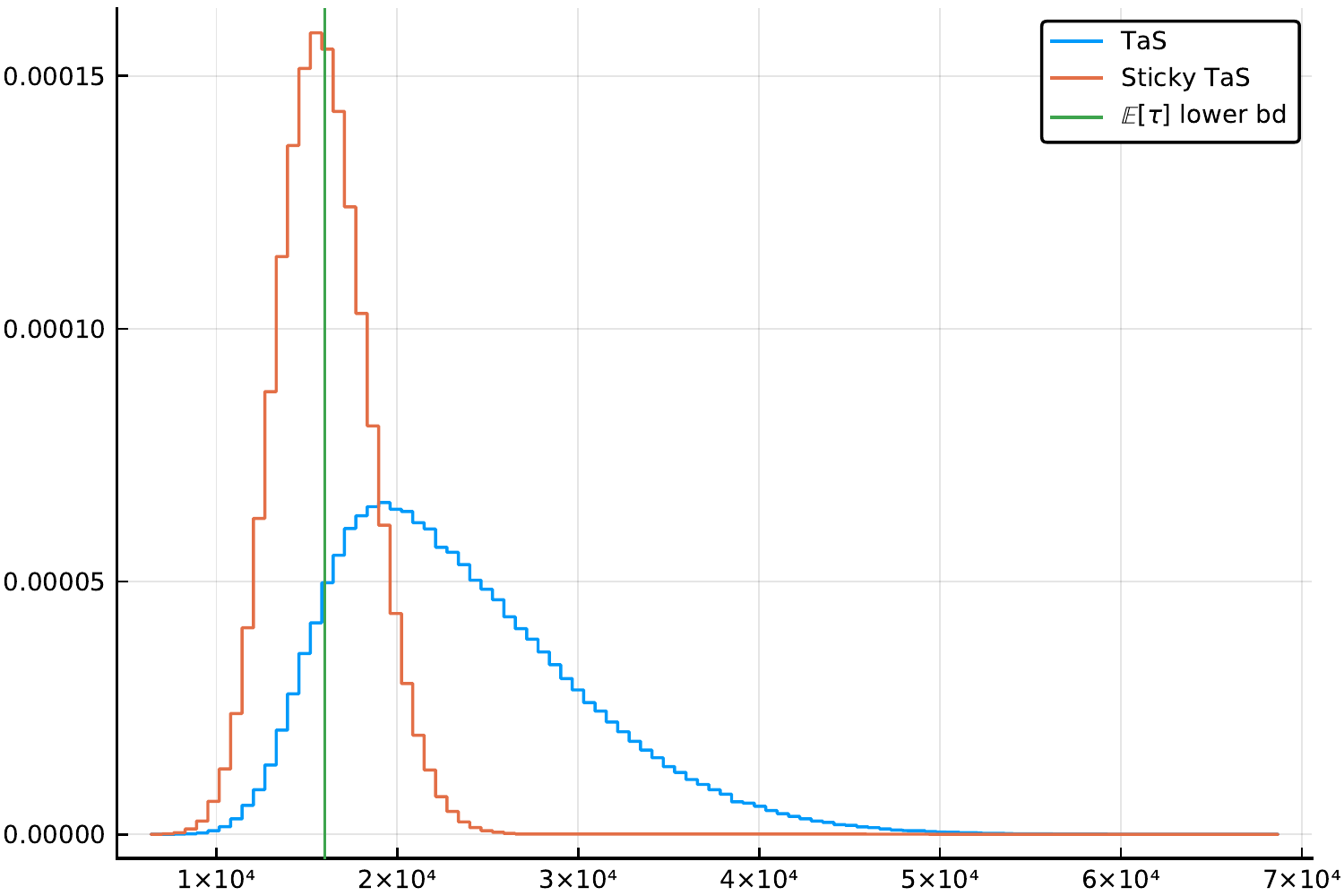}
  }
  \subfigure[Distance of $\bm{N}_\tau/\tau$ to $\w^*(\vmu)$ for TaS]{%
    \includegraphics[width=.5\textwidth]{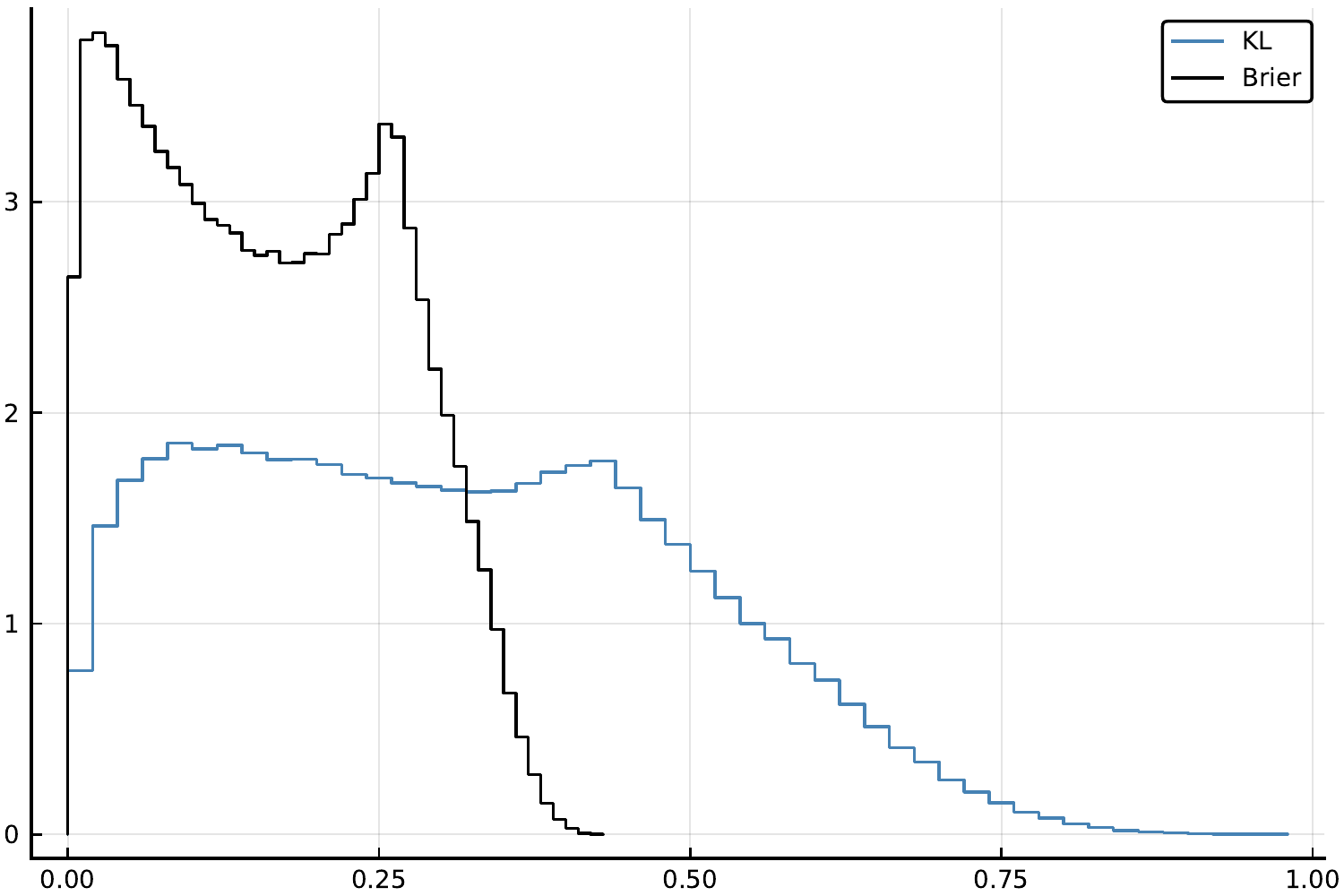}
  }
  \caption{Suboptimality of Track-and-Stop with $C$-tracking on a $K=10$ arm instance of the \textit{Any Half-Space} problem, where each $\mu_k = -1/10$. We run the algorithm with an excessively small $\delta = e^{-80}$ to focus on the asymptotic regime, using $500K$ repetitions. \textit{Left}: empirical distribution of the stopping time of Track-and-Stop and of Sticky Track-and-Stop, with resprctive means 24K and 16K. \textit{Right}: the reason for the suboptimality is that the sampling proportions of Track-and-Stop do not converge to the oracle weights.
  }\label{fig:TaS.fails}
\end{figure}

We argue that Track and Stop in general does not ensure convergence of $N_t/t$ to $\w^*(\vmu)$ when that set is not convex. We illustrate our claim on the \textit{Any Half-Space} problem, a generalisation of \emph{Any Sign} from Table~\ref{tab:problems}. Given $n\in\N$ hyperplanes of $\R^K$ passing through 0, parametrized for $m\in[n]$ by normal vectors $\u_m \in \R^K$ with $\Vert \u_m \Vert_1 = 1$, the algorithm has to return $(m,s) \in [n]\times\{-1,1\}$ such that $ s \vmu^\top \u_m \geq 0$. i.e. it must return any of the half-spaces in which $\vmu$ lies. See Figure~\ref{fig:halfspace}. The arms have Gaussian distributions with variance 1.

This problem is chosen for the simplicity of its $\w^*$ mapping. Indeed $\w^*(\vmu, \neg (m,s))) = \{\u_m\}$ if $s\vmu^\top \u_m\geq 0$ and $\w^*(\vmu, \neg (m,s))) =\triangle_K$ otherwise. The distance to that alternative is $D(\vmu, \neg (m,s)) = \mathbb{I}\{s\vmu^\top \u_m \geq 0\} (\vmu^\top \u_m)^2$. The optimal weights set $\w^*(\vmu)$ is the union of those $\{\u_m\}$ for which the distance is the greatest, and can be non-convex.

\begin{wrapfigure}[15]{r}{.3\textwidth}
  \centering
\begin{tikzpicture}[
  pt/.style={fill,circle,inner sep=.25ex},
  ]

  \fill[orange,semitransparent] (0,0) -- (-2,-2) -- (-2,1) -- cycle;
  \fill[yellow,semitransparent] (0,0) -- (-2,-2) -- (1,-2) -- cycle;

  \draw (-.5,1) -- (1,-2);
  \draw (1,-.5) -- (-2,1);
  \draw[dotted] (0,0) -- (-2,-2);
  \node at (.6,-1.2) [pin=0:{$\u_1$}] {};
  \node[pt,label=45:{$(0,0)$}] at (0,0) {};
  \coordinate (hmu) at (-.7,-1.4);
  \node[pt] at (hmu) [label=270:{$\hat\vmu_t$}] {};
  \node[pt] at (-1,-1) [label={$\vmu$}] {};
\end{tikzpicture}
\caption[Short]{\emph{Any Half-Space} problem.
  The oracle answers $i_F(\vmu)$ are
  \onebox{orange}~$\set{\u_1}$ and
  \onebox{yellow}~$\set{\u_2}$ (and both on the diagonal).
}\label{fig:halfspace}
\end{wrapfigure}
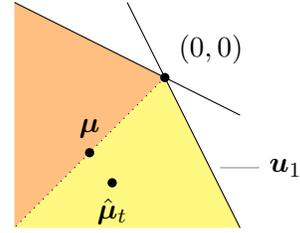

For $a\in[0,1]$, let $K=2$, $n=2$, $\u_1 = (a, 1-a)$, $\u_2=(1-a, a)$ and $\vmu = (\mu_0, \mu_0)$ for some $\mu_0\in\R$. Suppose that after stage $t_0$ of Track and Stop, $\hat{\vmu}_t$ verifies that $\hat{\vmu}_t^\top\u_m$ has same sign as $\vmu^\top\u_m$ for both $m\in\{1, 2\}$ (in expectation this happens except on at most a finite number of stages). Then $\w^*(\hat{\vmu}_t) = \{\u_1\}$ iff $\hat{\mu}_{t,1} > \hat{\mu}_{t,2}$ and $\w^*(\hat{\vmu}_t) = \{\u_2\}$ iff $\hat{\mu}_{t,1} < \hat{\mu}_{t,2}$. The case $\hat{\mu}_{t,1}=\hat{\mu}_{t,2}$ has probability 0, hence we ignore it.

C-tracking ensures that $N_t/t$ is close to $\sum_{s=1}^t \w_s$ for $\w_s\in\w^*(\hat{\vmu}_s)$. Calling $T_1(t)$ the number of stages up to $t$ for which $\hat{\mu}_{t,1}>\hat{\mu}_{t,2}$ and neglecting the first $t_0$ stages, $N_t/t \approx T_1(t) \u_1 + (1-T_1(t))\u_2$.  In the nomenclature of random walks, $T_1(t)$ is the occupation time of the region below the diagonal on Figure~\ref{fig:halfspace}.

In order for Track and Stop to be optimal, $N_t/t$ need to be close to $\w^*(\vmu) = \{\u_1, \u_2\}$ at $t_{opt} = \log(1/\delta)/D(\vmu)$. For $\u_1\neq \u_2$ this means that the distribution of $T_1(t_{opt})/t_{opt}$ must be concentrated on $\{0,1\}$. If the limit distribution of $T_1(t)/t$ (assuming it exists) for $t\to+\infty$ has mass in $(0,1)$, Track and Stop likely has suboptimal asymptotic sample complexity.

In the case of $a=1/2$, $\u_1=\u_2=(1/2,1/2)$, fot $t$ even $N_{t,1}=N_{t,2}$ and $T_1(t)=\#\{s\leq t: \sum_{u=1}^{s/2}(X_{2u}^{(1)} - X_{2u+1}^{(2)})>0\}$ is the occupation time of $\R^+$ for a Gaussian random walk. Its limit distribution is the Arcsine distribution. But in that case $N_t/t$ is always optimal. Experimentally, when $a\neq 1/2$ (hence $\u_1\neq \u_2$ and $N_t/t$ not always optimal), we observe that the limit distribution for $T_1(t)/t$ is not Arcsine, but has mass in $(0,1)$ for $a\in(0,1)$. See Figure~\ref{fig:pseudo_arcsine}.

Figure~\ref{fig:TaS.fails} displays the stopping time of Track and Stop on such an hyperplane problem for $K=n=10$ and shows that Track and Stop is empirically suboptimal and that $N_\tau/\tau$, proportions of pulls at the stopping time, is not concentrated near $\w^*(\vmu)$.

\begin{figure}
	\subfigure[Histogram of pulling proportion $N_{t,1}/t$ for TaS.]{
    \includegraphics[width=.5\textwidth]{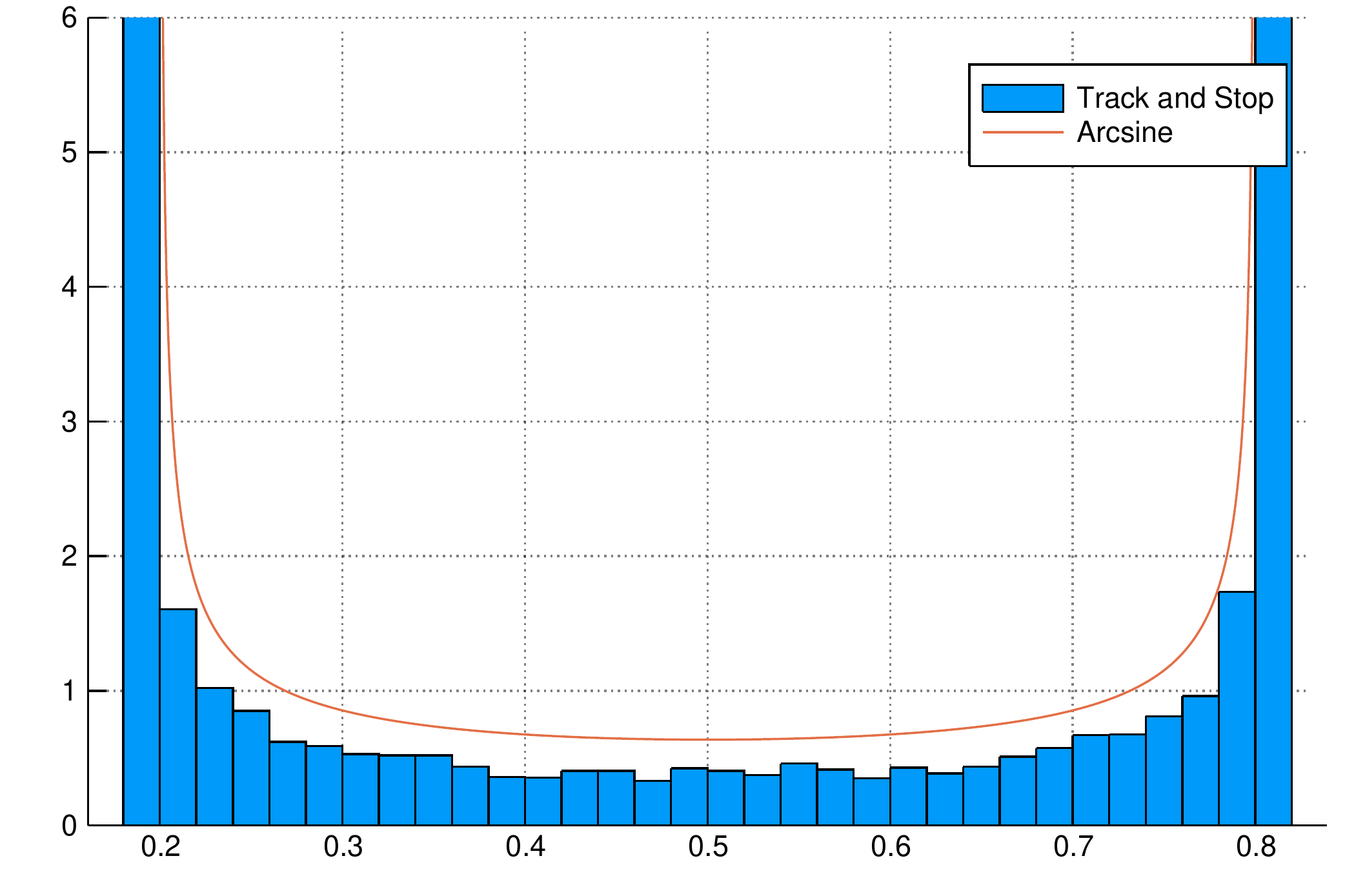}
  }
  \subfigure[Proportion of runs with $N_{t,1}/t$ in the interval $(1.01\times a,(1-a)/1.01)$ for TaS.]{
    \includegraphics[width=.5\textwidth]{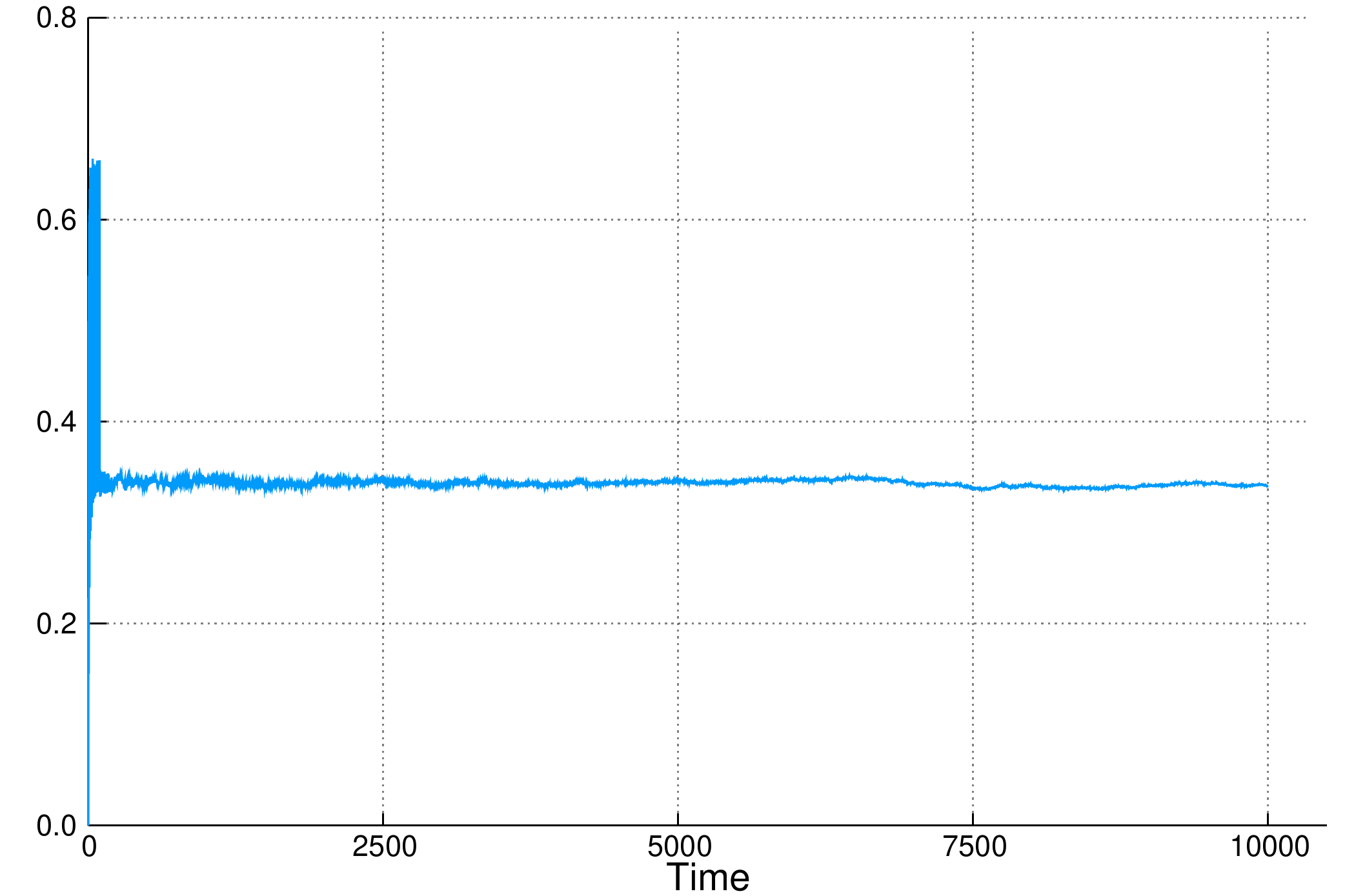}
  }
  \caption{Pulling proportions of Track-and-Stop with $C$-tracking on the \textit{Any Half-Space} problem with $K=2$ and $a=1/5$. Both figures use data from the same 10000 runs of Track-and-Stop. \textit{Left}: histogram at $T=10000$ of $N_{t,1}/t$, normalized such that the total area of the bars is 1. The values of the leftmost and rightmost bars are $\approx 15$. The probability distribution function of the Arcsine distribution is shown for comparison. \textit{Right}: evolution over time of the mass in the interval $(1.01\times a,\frac{1-a}{1.01})$ (non-extremal bars on the left).
  }\label{fig:pseudo_arcsine}
\end{figure}


\section{Conclusion}

We characterized the complexity of multiple-answers pure exploration bandit problems, showing a lower bound and exhibiting an algorithm with asymptotically matching sample complexity on all such problems. That study could be extended in several interesting directions and we now list a few.
\begin{itemize}[wide,nosep]
\item The computational complexity of Track-and-Stop is an important issue: it would be desirable to design a pure exploration algorithm with optimal sample complexity which does not need to solve a min-max problem at each step. Furthermore, the same would need to be done for the sticky selection of an answer for the multiple-answers setting.
\item Both lower bounds and upper bounds in this paper are asymptotic. In the upper bound case, only the forced exploration rounds are considered when evaluating the convergence of $\hat{\vmu}_t$ to $\vmu$, giving rise to potentially sub-optimal lower order terms. A finite time analysis with reasonably small $o(\log(1/\delta))$ terms for an optimal algorithm is desirable. In addition, while selecting one of the oracle answers to stick to has no asymptotic cost, it could have a lower order effect on the sample complexity and appear in a refined lower bound.
\item Current tools in the theory of Brownian motion are insufficient to characterise the asymptotic distribution of proportions induced by tracking, even for two arms. Without tracking the Arcsine law arises, so this slightly more challenging problem holds the promise of similarly elegant results.

\item Finally, the multiple answer pure exploration setting can be extended in various ways. Making $\mathcal I$ continuous leads to regression problems. The parametric assumption that the arms are in one-parameter exponential families could also be relaxed.
\end{itemize}

\newpage

\bibliography{bib}

\newpage
\appendix

\section{Notations}

\begin{table}[H]
  \begin{tabular}{lll}
    \textbf{Concept} & \textbf{Symbol}
    \\
    Exponential family mean parameter & $\mu \in \mathcal O$
    \\
    Bandit model & $\vmu\in\mathcal{M} \subseteq \mathcal{O}^K$
    \\
    Possible answers & $\mathcal{I}$
    \\
    Correct answer for bandit model $\vmu$ & $i^*(\vmu):\mathcal{M}\to 2^{\mathcal{I}}$
    \\
    Alternative to $i\in\mathcal{I}$ & $\neg i = \{\vmu\in\mathcal{M} : i\notin i^*(\vmu)\}$
    \\
    $K$-Simplex & $\triangle_K$
    \\
    Non-negative orthant & $\mathcal{Q}_+$
    \\
    interior, closure, convex hull of set $A$ & $\mathring{A}$, $\text{cl}(A)$, $\conv(A)$
    \\
    Oracle answers and weights & $i_F(\vmu):\mathcal{M}\to 2^{\mathcal{I}}$, $\w^*(\vmu):\mathcal{M} \to \mathcal{P}(\triangle_K)$
    \\
    Bandit arm & $k \in [K]$
    \\
    Number of samples of arm $k$ at time $t$ & $N_{t,k}$
    \\
    mean of samples of arm $k$ at time $t$ & $\hat{\mu}_{t,k}$
  \end{tabular}
\end{table}

\section{Lower bound Proofs}\label{pf:lbd}
In this section we build up to the proof of the lower bound Theorem~\ref{thm:lbd}. We start with the minimax result Lemma~\ref{lem:nash_finitely_supported}.

\subsection{Lemma~\ref{lem:nash_finitely_supported}: Characteristic time as the value of a game}
Let $\neg j$ be the set of bandit problems for which $j$ is not a valid answer. Let us define the characteristic time by
\[
  \frac{1}{T^*_j(\vmu)}
  ~=~
  D(\vmu, \neg j)
  ~=~
  \sup_\w \inf_{\vlambda \in \neg j} \sum_k w_k d(\mu_k, \lambda_k)
  .
\]
We now view the characteristic time problem as defining a two-player zero-sum game, with the purpose of obtaining a minimax optimal mixed strategy for the $\inf$ player. We will use such a mixed strategy to construct hard learning problems for proving sample complexity lower bounds in the next section. In this section we focus on the existence of the minimax strategy. We define the \emph{bandit complexity} game to be the semi-infinite two-player zero-sum simultaneous game where:
\begin{itemize}
\item MAX's pure strategies are arms $i \in [K] = \{1, \ldots, K\}$,
\item MIN's pure strategies are bandit models $\vlambda \in \neg j \subseteq \mathcal M$ (we may equivalently have MIN play a point $s \in S \df \setc*{\del[\big]{d(\mu_1, \lambda_1), \ldots, d(\mu_K, \lambda_K)}}{\vlambda \in \neg j} \subseteq [0,\infty)^K$),
\item the payoff function is $(i, \vlambda) \mapsto d(\mu_i, \lambda_i)$ (or, equivalently, $(i, s) \mapsto s_i$).
\end{itemize}
By definition, $D(\vmu, \neg j)$ is the optimal payoff when MAX randomises and plays first. We aim to show that a matching randomised strategy exists for when MIN plays first. That is, we want to establish a min-max theorem.
\begin{proof}[of Lemma~\ref{lem:nash_finitely_supported}]
  Combining (a) a standard application of Sion's minimax theorem to the bilinear function $f : \triangle \times \conv(S) \to \R$ defined by $f(\w, s) = \tuple{\w,s}$ and (b) the support size insight of \citet[Theorem~2.4.2]{BG54} yields the Lemma.
\end{proof}

For convenience, we will assume in the remainder that the infimum above is attained (e.g.\ when the convex hull of $S$ is compact), possibly on the closure of $\neg j$. If it is not, we need to apply the below arguments to a sequence of $\epsilon$-suboptimal $\pr$ and let $\epsilon \to 0$. At any rate, we assume there exist  $\vlambda^1,\ldots,\vlambda^K \in \neg j$ (or its closure) and $\q\in \triangle_K$ such that
\begin{equation}\label{eq:saddle}
  \forall i :
  \sum_k q_k d(\mu_i, \lambda^k_i)
  ~\le~
  D(\vmu, \neg j)
  .
\end{equation}


\subsection{Consequences of the Minimax result}
In this section we build up lower bounds by relating the probability of any event between two or more bandit problem.
We start with a useful change of measures observation, used in \citep{eps.del.PAC.bai} to derive a lower bound on the sample complexity of $\varepsilon$-Best Arm Identification.
\begin{proposition}
Consider two distributions $\pr$ and $\mathbb Q$.
Let us denote the log-likelihood ratio after $n$ rounds by $L_n = \ln \frac{\dif \pr}{\dif \qr}$. Then for any measurable event $A \in \mathcal F_n$ and threshold $\gamma \in \mathbb R$,
\begin{equation}\label{eq:com}
  \pr(A)
  ~\le~
  e^\gamma \qr(A)
  +
  \pr\set*{L_n > \gamma}
  .
\end{equation}
\end{proposition}
\begin{proof}
  \begin{align*}
    \qr(A) = \ex_{\pr}[\mathbb{I}_A e^{-L_n}]
    &\geq \ex_{\pr}[\mathbb{I}_{A\cap\{L_n\leq \gamma\}} e^{-L_n}]\\
    &\geq \pr(A\cap\{L_n\leq \gamma\}) e^{-\gamma}
    \leq e^{-\gamma} (\pr(A) - \pr\{L_n > \gamma\}) \: .
  \end{align*}

\end{proof}

\subsection{Likelihood ratio Martingales}

Next we investigate the specific form of the likelihood ratio between two bandit models. Fix bandit models $\vmu$ and $\vlambda$, and any sampling strategy. Then after $n$ rounds,
\begin{align*}
  \ln \frac{\dif \pr_\vmu}{\dif \pr_\vlambda }
  &~=~
  \sum_i N_{n,i} \KL(\nu_{\mu_i,i}, \nu_{\lambda_i,i})
  + M_n(\vmu, \vlambda)
\end{align*}
where $ M_n(\vmu, \vlambda)$ is a martingale. To see this, we write $\KL(\nu_{\mu,i}, \nu_{\lambda,i}) = d(\mu, \lambda) = \phi(\mu) - \phi(\lambda) - (\mu-\lambda) \phi'(\lambda)$, where we write $\phi$ for the convex generator of the Bregman divergence $d(\cdot, \cdot)$. Then
\begin{align*}
  \ln \frac{\dif \pr_\vmu}{\dif \pr_\vlambda }
  &~=~
  \sum_i N_{n,i}
  \del*{
    d(\hat \mu_{n,i}, \lambda_i)
    - d (\hat \mu_{n,i}, \mu_i)
    }
  \\
  &~=~
  \sum_i N_{n,i}
    \del[\Big]{
    d(\mu_i, \lambda_i)
    + \del[\big]{\phi'(\mu_i) - \phi'(\lambda_i)} \del[\big]{\hat \mu_{n,i} - \mu_i}
    }
\end{align*}
hence $M_n(\vmu, \vlambda) = \sum_i N_{n,i} \del[\big]{\phi'(\mu_i) - \phi'(\lambda_i)} \del[\big]{\hat \mu_{n,i} - \mu_i}$.

\subsection{Exploiting the Minimax Distribution}
We now bound the probability of any event between $\vmu$ and the hard problems given by the minimax distribution.

\begin{lemma}\label{lem:martingaled}
  Fix a bandit model $\vmu$ with sub-Gaussian arm distributions. Let $\q$ and $\vlambda^1, \ldots, \vlambda^K$ be a minimax witness from Lemma~\ref{lem:nash_finitely_supported}, and let us introduce the abbreviation $\alpha_i = \phi'(\mu_i) - \sum_k q_k \phi'(\lambda^k_i)$. Fix sample size $n$, and consider any event $A \in \mathcal F_n$. Then for any $\beta > 0$
\[
  \max_{k \in [K]}~ \pr_{\vlambda^k}\set{A}
  ~\ge~
  e^{-\frac{n}{T^*(\vmu)} - \beta}
  \del*{\pr_\vmu\set*{A} - \exp \del*{\frac{-\beta^2}{2 n \max_i \alpha_i^2}}}
.
\]
\end{lemma}
In words, if $A$ is likely under $\vmu$ the it must also be likely under at least one $\vlambda^k$ for sample sizes $n \ll T^*(\vmu)$.

\begin{proof}
Let us form the (Bayesian) mixture distribution $\pr_\q = \sum_k q_k \pr_{\vlambda^k}$. We have
\[
  L_n
  ~=~
  -\ln \frac{\dif \pr_\q}{\dif \pr_\vmu}
  ~\le~
  \sum_k q_k
  \ln \frac{\dif \pr_\vmu}{\dif \pr_{\vlambda^k} }
  .
\]
It follows that for any $\gamma \in \R$ we have
\begin{align}
  \notag
  \set*{L_n > \gamma}
  &~\subseteq~
    \set*{
    \sum_k q_k \sum_i N_{n,i}d(\mu_i, \lambda^k_i) + \sum_k q_k M_n(\vmu, \vlambda^k)
    > \gamma}
  .
\intertext{%
  Picking $\gamma = \frac{n}{T^*(\vmu)} + \beta$, we find
    }
    \notag
&  ~=~
    \set*{
    \sum_k q_k \sum_i N_{n,i}d(\mu_i, \lambda^k_i) + \sum_k q_k M_n(\vmu, \vlambda^k)
    > \frac{n}{T^*(\vmu)}  + \beta}
  %
    \intertext{%
    Since $(\w^*,\q)$ is a Nash equilibrium of the game and $\bm{N}_n/n$ is a mixed strategy for the first player, $\sum_k q_k \sum_i N_{n,i} d(\mu_i, \lambda^k_i) \leq n \sum_k q_k \sum_i w^*_i d(\mu_i, \lambda^k_i) = \frac{n}{T^*(\vmu)}$, so we find}
    \notag
  &  ~\subseteq~
    \set*{
    \sum_k q_k M_n(\vmu, \vlambda^k)
    >  \beta}
  \\
  \notag
  &  ~=~
    \set*{
    \sum_k q_k
    \sum_i N_{n,i} \del[\big]{\phi'(\mu_i) - \phi'(\lambda^k_i)} \del[\big]{\hat \mu_{n,i} - \mu_i}
    >  \beta}
  \\
  \label{eq:crucial.mart}
  &~=~
    \set*{
    \sum_i N_{n,i} \alpha_i \del[\big]{\hat \mu_{n,i} - \mu_i}
    >  \beta}
\end{align}
The above left-hand quantity is a martingale of length $n$. Using the sub-Gaussianity assumption, the Hoeffding-Azuma inequality gives
\[
  \pr_\vmu \set*{L_n > \gamma}
  ~\le~
  \exp \del*{\frac{-\beta^2}{2 n \max_i \alpha_i^2}}
  .
\]
%
Let $A$ be a $\mathcal{F}_n$-measurable event. Combination with the change of measure argument \eqref{eq:com} with  $\max_k~ \pr_{\vlambda^k}\set{A}
  \ge
  \pr_\q \set{A}$ gives the result.
\end{proof}

The sub-Gaussian assumption of the Lemma can undoubtedly be relaxed. The crucial requirement is that the $n$-step martingale in \eqref{eq:crucial.mart} concentrates, and hence cannot be large w.h.p.

We are now ready for the proof of Theorem~\ref{thm:lbd}, in which we will carefully tune the time $n$ to which we apply the above result.

\subsection{Proof of Theorem~\ref{thm:lbd}}

We will bound the expectation of the stopping time $\tau_\delta$ through Markov's inequality. For $T>0$,
\begin{align*}
\ex_\vmu[\tau_\delta]
&\geq T (1 - \pr_\vmu(\tau_\delta \leq T)) \: .
\end{align*}
The event $\{\tau_\delta \leq T\}$ can be partitioned depending on the answer which is returned. Since the algorithm is $\delta$-PAC by hypothesis, $\pr_\vmu(\tau_\delta \leq T, \ihat_\delta \notin i^*(\vmu)) \leq \delta$. So
\begin{align*}
\pr_\vmu(\tau_\delta \leq T)
&= \sum_i \pr_\vmu(\tau_\delta \leq T, \ihat_\delta = i)\\
&\leq \delta + \sum_{i\in i^*(\vmu)} \pr_\vmu(\tau_\delta \leq T, \ihat_\delta = i) \: .
\end{align*}
For $i\in i^*(\vmu)$, fix a minimax strategy $\vlambda^1, \ldots, \vlambda^K \in \neg i$ and $\q \in \triangle_K$ as given by Lemma~\ref{lem:equilibrium}. Then by Lemma~\ref{lem:martingaled}, for any $\beta>0$
\begin{align*}
\pr_\vmu(\tau_\delta \leq T, \ihat_\delta = i)
&\leq \exp\left( \frac{T}{T^*_i(\vmu)} + \beta \right) \max_k\pr_{\vlambda^k}(\tau_\delta \leq T, \ihat_\delta = i)
	+ \exp \del*{\frac{-\beta^2}{2 T \max_i \alpha_i^2}}\\
&\leq \delta \exp\left( \frac{T}{T^*_i(\vmu)} + \beta \right)
	+ \exp \del*{\frac{-\beta^2}{2 T \max_i \alpha_i^2}}
\end{align*}
Let $\alpha^2 = \max_i \alpha_i^2$. For $\eta \in (0,1)$, $T \leq (1-\eta) T^*_i(\vmu) \log(1/\delta)$ and $\beta =  \frac{\eta}{2\sqrt{1-\eta}}\sqrt{ \frac{T}{T_i^*(\vmu)} \log(1/\delta)} $,

\begin{align*}
\pr_\vmu(\tau_\delta \leq T, \ihat_\delta = i)
&\leq \delta \exp\left(
		\frac{T}{T^*_i(\vmu)}
		+ \frac{\eta}{2\sqrt{1-\eta}}\sqrt{ \frac{T}{T_i^*(\vmu)} \log(1/\delta)}
	\right)
	+ \exp \del*{\frac{-\eta^2 \log(1/\delta)}{8 (1-\eta) T_i^*(\vmu)\alpha^2}}\\
&\leq \delta \exp \del*{(1-\eta/2)\log\frac{1}{\delta}}
	+ \exp \del*{\frac{-\eta^2 \log(1/\delta)}{8 (1-\eta) T_i^*(\vmu) \alpha^2}}\\
&= \delta^{\eta/2} + \delta^{\eta^2 / (8 (1-\eta) T_i^*(\vmu)\alpha^2)}\\
&\xrightarrow[\delta \to 0]{} 0
\end{align*}
Suppose that $T = \min_{i\in i^*(\vmu)} (1-\eta)T_i^*(\vmu) \log(1/\delta)$ for some $\eta\in(0,1)$. Then we must have $\lim_{\delta\to 0} \pr_\vmu(\tau_\delta \leq T) = 0 $
and therefore
$\liminf_{\delta\to 0}\frac{\ex_\vmu[\tau_\delta]}{\log(1/\delta)} \geq \lim_{\delta\to 0} \frac{T}{\log(1/\delta)}(1-\pr_\vmu(\tau_\delta \leq T)) = (1-\eta) \min_{i\in i^*(\vmu)} T_i^*(\vmu)$.
Letting $\eta$ go to zero, we obtain that
\begin{align*}
\liminf_{\delta\to 0}\frac{\ex_\vmu[\tau_\delta]}{\log(1/\delta)} \geq \min_{i\in i^*(\vmu)} T_i^*(\vmu) \: .
\end{align*}

\section{Failure of D-tracking}\label{sec:failure_D_tracking}

\newcommand{\yes}{\text{yes}}
\newcommand{\no}{\text{no}}

We illustrate the suboptimality of Track-and-Stop with D-tracking in a single-answer problem. On a 3-arms problem with Gaussian distributions with same variance, the algorithm must answer the two following queries:
\begin{itemize}
\item What is the sign of $\mu_1$?
\item Is there $m\in\{1,2\}$ such that $\vmu^\top \u^{(m)} \geq 0$? For some $a\in(0,1)$, $\u^{(1)} = (0, a, 1-a)$ and $\u^{(2)} = (0, 1-a, a)$.
\end{itemize}
The possible answers are $\{(+,\yes), (+,\no), (-,\yes), (-,\no)\}$, and there is only one correct answer.

The divergence from $\vmu$ to an alternative where the sign is flipped is $D_1(\vmu) = \frac{1}{2}\mu_1^2$. If $\vmu^\top \u^{(m)} < 0$ for both $\u^{(m)}$, then the divergence from $\vmu$ to an alternative where the second answer is yes is $D_{2,3}=\max(\frac{1}{2}(\vmu^\top \u^{(1)})^2, \frac{1}{2}(\vmu^\top \u^{(2)})^2)$. Let $D_2(\vmu)$ and $D_3(\vmu)$ be the two terms in that maximum. In that case, the oracle weights $\w^*(\vmu)$ are
\begin{itemize}
\item $\{\w^{(1)}(\vmu)\}$ if $D_2>D_3$, where $\w^{(1)}(\vmu) = (\frac{D_2}{D_1+D_2}, a \frac{D_1}{D_1+D_2}, (1-a) \frac{D_1}{D_1+D_2})$,
\item $\{\w^{(2)}(\vmu)\}$ if $D_2<D_3$, where $\w^{(2)}(\vmu) = (\frac{D_3}{D_1+D_3}, (1-a) \frac{D_1}{D_1+D_3}, a \frac{D_1}{D_1+D_3})$,
\item the convex hull of $\{\w^{(1)}(\vmu),\w^{(2)}(\vmu)\}$ if $D_2=D_3$.
\end{itemize}
Let $\vmu = (-\mu_0, -\mu_0, -\mu_0)$ for $\mu_0>0$. Then $\w^*(\vmu) = \conv\{(\frac{1}{2}, \frac{1}{2}a, \frac{1}{2}(1-a)), (\frac{1}{2}, \frac{1}{2}(1-a), \frac{1}{2}a)\}$.

If $a\in\{0,1\}$, then these oracle weights correspond to $\w^{(1)}$ and $\w^{(2)}$ defined below Remark~\ref{rem:D_tracking}.

Track-and-Stop with D-tracking will track $\w^{(1)}(\hat{\vmu}_t)$ or $\w^{(2)}(\hat{\vmu}_t)$ depending on which side of the hyperplane $D_2=D_3$ the empirical mean $\hat{\vmu}_t$ lies. As in section~\ref{sec:failure_tas}, we do not know the distribution of the tracked vector $\w_t$, but if $\hat{\vmu}_t$ crosses that boundary often enough, then as explained below Remark~\ref{rem:D_tracking}, D-tracking will get $N_t/t$ outside of the convex hull of $\{\w^{(1)}(\vmu),\w^{(2)}(\vmu)\}$ and be suboptimal.

We verify experimentally that suboptimality in Figure~\ref{fig:D_tracking}. For these experiments, the Gaussians have variance $1/4$, $\mu_0 = 1/5$, $a=1/10$. The fixed optimal sampling strategy samples $\argmin_k N_{t,k} - t(\frac{1}{2}, \frac{1}{2}a, \frac{1}{2}(1-a))_k$.

\begin{figure}
  \centering
  \includegraphics[width=.5\textwidth]{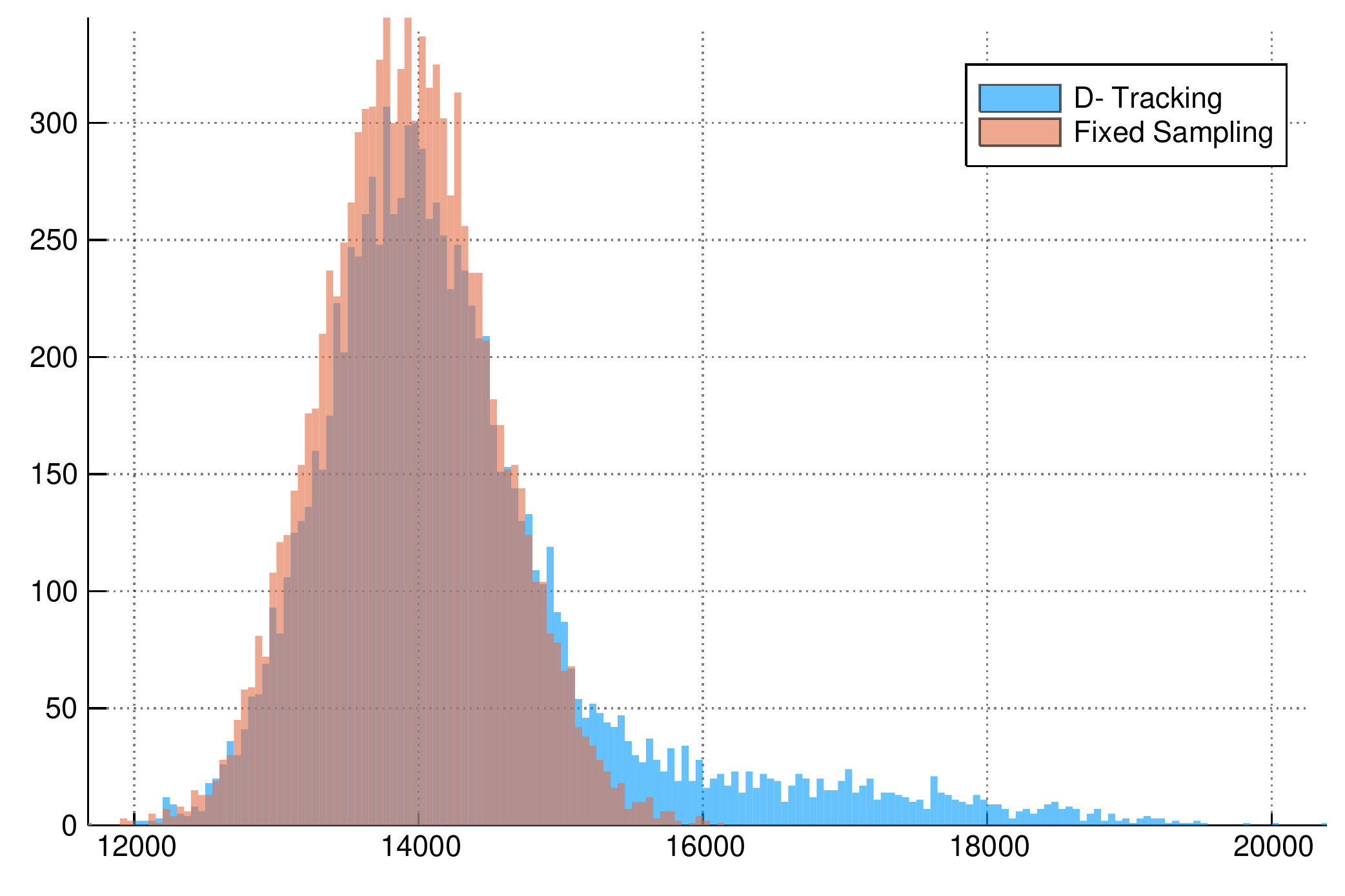}
  \caption{
    Histogram of the stopping time of Track-and-Stop with D-tracking and of fixed optimal sampling. The data is comprised of 10000 runs of each algorithm, with $\delta=e^{-60}$.
  }
  \label{fig:D_tracking}
\end{figure}

\section{Continuity Proofs}\label{sec:continuity_proof}

We first introduce the necessary notions used in the modification of Berge's theorem we will apply, following \citep{FKV14}.

\begin{definition}
For a function $f:U\to\R$ with $U$ a non-empty subset of a topological space, define the level sets
\begin{align*}
L_f(y, U) &= \{x\in U \: : \: f(x)\leq y\} \: ,\\
L_f^<(y, U) &= \{x\in U \: : \: f(x)< y\} \: .
\end{align*}
A function $f$ is \textit{lower semi-continuous} on $U$ if all the level sets $L_f(y, U)$ are closed. It is \textit{inf-compact} on $U$ if all these level sets are compact. It is \textit{upper semi-continuous} if all the strict level sets $L_f^<(y, U)$ are open.
\end{definition}

Let $\mathbb{X}$ and $\mathbb{Y}$ be Hausdorff topological spaces. Let $u:\mathbb{X}\times \mathbb{Y}\to \R$ be a function, $\Phi:\mathbb{X}\to \mathbb{S}(\mathbb{Y})$ be a set-valued function, where $\mathbb{S}(\mathbb{Y})$ is the set of non-empty subsets of $\mathbb{Y}$. The objects of study are
\begin{align*}
v(x) &= \inf_{y\in \Phi(x)} u(x, y) \: ,\\
\Phi^*(x) &= \{y\in\Phi(x)\: : \: u(x, y) = v(x)\} \: .
\end{align*}

For $U\subset \mathbb{X}$, let the graph of $\Phi$ restricted to $U$ be $Gr_U(\Phi) = \{(x,y)\in U\times \mathbb{Y} \: : \: y\in\Phi(x)\}$ .

\begin{definition}
A function $u:\mathbb{X}\times \mathbb{Y}\to \overline{\R}$ is called $\mathbb{K}$-inf-compact on $Gr_{\mathbb{X}}(\Phi)$ if for all non-empty compact subset $C$ of $\mathbb{X}$, $u$ is inf-compact on $Gr_C(\Phi)$.
\end{definition}

We will use two versions of Berge's theorem. The first one restricts $\Phi$ to be compact-valued. The second one removes that hypothesis on $\Phi$ at the price of hypotheses on $u$. Denote by $\mathbb{K}(\mathbb{X})$ the subset of $\mathbb{S}(\mathbb{X})$ containing non-empty compact subsets of $\mathbb{X}$.

\begin{theorem}[Berge's theorem]\label{th:berge}
Let $\mathbb{X}$ and $\mathbb{Y}$ be Hausdorff topological spaces. Assume that
\begin{itemize}
\item $\Phi: \mathbb{X} \to \mathbb{K}(\mathbb{X})$ is continuous (i.e. both lower hemicontinuous and upper hemicontinous),
\item $u: \mathbb{X} \times \mathbb{Y} \to \R$ is continuous.
\end{itemize}
Then the unction $v:\mathbb{X} \to \R$ is continuous and the solution multifunction $\Phi^*:\mathbb{X}\to\mathbb{S}(\mathbb{Y})$ is upper hemicontinuous and compact valued.
\end{theorem}

\begin{theorem}[{\citealt{FKV14}}]\label{th:berge_feinberg}
Assume that
\begin{itemize}
\item $\mathbb{X}$ is compactly generated,
\item $\Phi:\mathbb{X}\to \mathbb{S}(\mathbb{Y})$ is lower hemicontinuous,
\item $u:\mathbb{X}\times \mathbb{Y} \to \R$ is $\mathbb{K}$-inf-compact and upper semi-continuous on $Gr_\mathbb{X}(\Phi)$. 
\end{itemize}
Then the function $v:\mathbb{X} \to \R$ is continuous and the solution multifunction $\Phi^*:\mathbb{X}\to\mathbb{S}(\mathbb{Y})$ is upper hemicontinuous and compact valued.
\end{theorem}

Theorem~\ref{th:continuity} is cut into several successive lemma, whose proofs together prove the theorem. The first three lemmas prove the continuity of $(\w, \vmu) \mapsto D(\w, \vmu, \neg i)$, first in the case where $\neg i$ is compact, then in the general case.

\begin{lemma}
Set $i\in\mathcal{I}$. If $\neg i$ is compact, then $(\w, \vmu)\mapsto D(\w, \vmu, \neg i)$ is jointly continuous on $\triangle_K\times \mathcal{M}$ and $\vlambda^*(\w, \vmu)$ is non-empty, upper hemicontinuous and compact valued.
\end{lemma}
\begin{proof}
We apply Theorem~\ref{th:berge} to
\begin{itemize}
\begin{minipage}{0.4\linewidth}
\item $\mathbb{X} = \triangle_K\times\mathcal{M}$,
\item $\mathbb{Y} = \neg i$,
\end{minipage}
\begin{minipage}{0.4\linewidth}
\item $\Phi(\vmu) = \neg i$,
\item $u((\w, \vmu),\vlambda) = D(\w, \vmu, \vlambda)$.
\end{minipage}
\end{itemize}
$\Phi$ is compact-valued, non-empty and continuous (since it is constant). $u$ is continuous. The hypotheses are verified and the theorem gives the wanted result.
\end{proof}







Let $\mathcal{Q}_+ = \{\w\in\R^K \: : \: \forall k\in[K], w_k\geq 0\}$, $\mathring{\mathcal{Q}}_+$ be its interior and for $\varepsilon>0$,  $\mathcal{Q}_+^\varepsilon = \{\w\in\R^K \: : \: \forall k\in[K], w_k\geq \varepsilon\}$.

\begin{lemma}
The function $(\w, \vmu) \mapsto D(\w, \vmu, \neg i)$ is jointly continuous on $\mathring{\mathcal{Q}}_+ \times \mathcal{M}$ .On the same set, $\vlambda^*(\w, \vmu)$ is upper hemicontinuous, non-empty and compact valued. In particular, the same properties hold on $\mathring{\triangle}_K \times \mathcal{M}$.
\end{lemma}
\begin{proof}
Let $\varepsilon>0$. We prove the result for $(\w, \vmu) \in \mathcal{Q}_+^\varepsilon\times \mathcal{O}^K$ (and note that $\mathcal{M}\subseteq \mathcal{O}^K$).

We will apply Theorem~\ref{th:berge_feinberg} to
\begin{itemize}
\begin{minipage}{0.4\linewidth}
\item $\mathbb{X} = \mathcal{Q}_+^\varepsilon\times\mathcal{O}^K$ ,
\item $\mathbb{Y} = \mathcal{O}^K$ ,
\item $\Phi((\w,\vmu)) = \neg i$ ,
\end{minipage}
\begin{minipage}{0.4\linewidth}
\item $u((\w, \vmu), \vlambda) = D(\w, \vmu, \vlambda)$ ,
\item $v(\w, \vmu) =  D(\w, \vmu, \neg i)$ .
\end{minipage}
\end{itemize}
We now verify the hypothesis of the theorem. First, $\mathbb{X}$ is compactly generated since it is a metric space. Secondly, $\Phi$ is lower hemicontinuous since it is constant.

The function $u$ is continuous, hence upper semi-continuous. It remains to check that $u$ is $\mathbb{K}$-inf-compact on $Gr_\mathbb{X}(\Phi)$.

Let $C$ be a non-empty compact subset of $\mathcal{Q}_+^\varepsilon \times \mathcal{O}^K$. We need to prove that $u$ is inf-compact on $Gr_C(\Phi) = C\times \neg i$ . The level sets $L_u(y,C\times \neg i)$ for $y\in\R$ are closed by continuity of $u$. Indeed they are the reverse image of a closed set $[-\infty, y]$ by a continuous function, hence they are closed in $(\mathcal{Q}_+^\varepsilon\times\mathcal{O}^K)\times \mathcal{O}^K$. We only need to prove that they are bounded.

Set $y\in\R$. Let $\mu^+_k = \sup_{(\w,\vmu) \in C} \mu_k$, finite since $C$ is compact. Define $\mu^-_k$ in a similar way with an infimum. For $j\in[K]$, $\lim d(\mu_j, \lambda_j) = +\infty$ when $\lambda_j$ approaches the boundaries of the open interval $\mathcal O$. Then for all $k\in[K]$, there exists $\lambda_k^+$ such that $\lambda>\lambda_k^+ \Rightarrow \forall (\w,\vmu) \in C, d(\mu_k, \lambda) > y/\varepsilon$. Define $\lambda_k^-$ in a similar way. For $\vlambda \notin [\lambda_1^-, \lambda_1^+]\times\ldots\times[\lambda_K^-, \lambda_K^+]$ there exists a $k\in[K]$ such that $d(\mu_k, \lambda_k) > y/\varepsilon$ for all $(\w,\vmu)\in C$. Since $w_k\geq \varepsilon$ for all $k$, for all $(\w,\vmu)\in C$, $D(\w,\vmu,\vlambda) = \sum_{k=1}^K w_k d(\mu_k, \lambda_k) > y$. The level set $L_u(y, C\times \neg i)$ is bounded.

We have verified the hypotheses of the theorem for compacts subsets of $\mathcal{Q}^\varepsilon_+\times \mathcal{O}^K$ and obtain that $v(\w, \vmu) = D(\w, \vmu, \neg i)$ is continuous as a function of $(\w, \vmu)$ on that set. On that same set, the function giving the points realizing the infimum $\vlambda^*(\w, \vmu)$ is upper hemicontinuous, non-empty and compact valued.
\end{proof}

For a projection $P$ on a subset of coordinates $S\subseteq [K]$ and $\w\in\mathcal{Q}_+$, denote the projected vector by $P\w$. The same proof as the one of the previous lemma applied to the projected spaces gives the following corollary.

\begin{corollary}
Let $P$ be a projection on a subset of coordinates $S\subseteq [K]$. Then the function $(\boldsymbol{u}, \vmu) \mapsto D(\boldsymbol{u}, \vmu, P\neg i)$ is continuous on $\mathring{P\mathcal{Q}}_+ \times \mathcal{M}$.
\end{corollary}

\paragraph{We now extend the continuity of $D(\w, \vmu, \neg i)$ on all of $\triangle_K\times\mathcal{M}$.}

\begin{lemma}\label{lem:D_w_mu_continuous}
The function $(\w, \vmu) \mapsto D(\w,\vmu, \neg i)$ is continuous on $\triangle_K \times \mathcal{M}$.
\end{lemma}
\begin{proof}


Let $(\w, \vmu) \in \triangle_K \times \mathcal{M}$ and $\vlambda^\varepsilon$ be such that $D(\w, \vmu, \vlambda) \leq D(\w, \vmu, \neg i) + \varepsilon$, which exists by definition of $D(\w, \vmu, \neg i)$ as an infimum.

The function $(\w, \vmu) \mapsto D(\w, \vmu, \vlambda^\varepsilon)$ is continuous. Hence, there exists $\varepsilon'$ such that for $\Vert \w' - \w \Vert_\infty \leq \varepsilon'$ and $\Vert \vmu' - \vmu \Vert_\infty \leq \varepsilon'$, $D(\w', \vmu', \vlambda^\varepsilon) \leq D(\w, \vmu, \vlambda^\varepsilon)+\varepsilon$. For $(\w', \vmu')$ in such a neighbourhood of $(\w, \vmu)$,
\begin{align*}
D(\w', \vmu', \neg i) \leq D(\w', \vmu', \vlambda^\varepsilon) \leq D(\w, \vmu, \vlambda^\varepsilon)+\varepsilon \leq D(\w, \vmu, \neg i) + 2\varepsilon \: .
\end{align*}
We proved that $(\w, \vmu) \mapsto D(\w, \vmu, \neg i)$ is upper semi-continuous.

Now let $\varepsilon>0$ be such that $\min_{k:w_k>0} w_k \geq 2\varepsilon$ and $(\w', \vmu')$ be in an $\varepsilon$ neighbourhood of $(\w,\vmu)$. Then $w_k>0 \Rightarrow w_k' > \varepsilon$. For $\mathbf{u}\in \triangle_K$, denote by  $P\mathbf{u}$ its projection on the coordinates for which $w_k>0$.
\begin{align*}
D(\w', \vmu', \neg i) &\geq D(P\w', \vmu', P\neg i)\: ,\\
D(\w, \vmu, \neg i) &= D(P\w, \vmu, P\neg i) \: .
\end{align*}
By the previous lemma, $(P\w',\vmu') \mapsto D(P\w', \vmu', P\neg i)$ is continuous on $P\mathcal{Q}_+^\varepsilon\times \mathcal{M}$. Hence for $(\w', \vmu')$ in a small enough neighbourhood of $(\w, \vmu)$, $D(P\w', \vmu', P\neg i) \geq D(P\w, \vmu, P\neg i) - \varepsilon$. In that neighbourhood,
\begin{align*}
D(\w', \vmu', \neg i)
&\geq D(P\w', \vmu', P\neg i)
\geq D(P\w, \vmu, P\neg i) - \varepsilon
\geq D(\w, \vmu, \neg i) - \varepsilon \: .
\end{align*}
This proves the lower semi-continuity of $(\w, \vmu) \mapsto D(\w, \vmu, \neg i)$. We now have both lower and upper semi-continuity: that function is continuous.

\end{proof}


We proved continuity of the function of $\triangle_K\times \mathcal{M}$ but upper hemi-continuity of $\vlambda^*$ only for $\w\in\mathring{\triangle}_K$. We now show an example where$\vlambda^*$ is empty at a $\w$ on the boundary of the simplex. 

Let $\vmu = (0,0)$, $\neg i = \{ (\sqrt{x}, \sqrt{1+1/x}) \: : \: x\in\R^+\}$ and $d(\mu_i, \lambda_i) = (\mu_i-\lambda_i)^2$. Then
\begin{align*}
\inf_\vlambda [w_1 d(\mu_1, \lambda_1) + w_2 d(\mu_2, \lambda_2)]
&= \inf_x(w_1 x + (1-w_1)(1+\frac{1}{x})) = 1 - w_1 + 2\sqrt{w_1(1-w_1)} \: ,
\end{align*}
with minimum attained for $w_1\in(0,1)$ at $x = \sqrt{\frac{1-w_1}{w_1}}$. Hence for $w_1>0$, $\vlambda^*(\w,(0,0)) = \{(\sqrt{\frac{1-w_1}{w_1}}, \sqrt{\frac{w_1}{1-w_1}})\}$, non-empty and compact.
If $w_1 = 0$ then the infimum is not attained and $\vlambda^*$ is empty.

\begin{lemma}\label{lem:D_mu_continuous}
For all $i\in\mathcal{I}$, $D(\vmu, \neg i)$ is a continuous function of $\vmu$ on $\mathcal{M}$ and $\w^*(\vmu, \neg i)$ is upper hemicontinuous and compact-valued.
\end{lemma}
\begin{proof}
We apply Theorem~\ref{th:berge} to
\begin{itemize}
\begin{minipage}{0.4\linewidth}
\item $\mathbb{X} = \mathcal{M}$,
\item $\mathbb{Y} = \triangle_K$,
\end{minipage}
\begin{minipage}{0.4\linewidth}
\item $\Phi(\vmu) = \triangle_K$,
\item $u(\vmu, \w) = D(\w, \vmu, \neg i)$.
\end{minipage}
\end{itemize}
$\Phi$ is compact-valued, non-empty and continuous (since it is constant). $u$ is continuous by Lemma~\ref{lem:D_w_mu_continuous}. The hypotheses are verified and the theorem gives the wanted result.
\end{proof}

\begin{lemma}
For all $i\in\mathcal{I}$, $D(\vmu)$ is a continuous function of $\vmu$ on $\mathcal{M}$ and $\w^*(\vmu)$ is upper hemicontinuous and compact-valued.
\end{lemma}
\begin{proof}
We apply Theorem~\ref{th:berge} to
\begin{itemize}
\begin{minipage}{0.4\linewidth}
\item $\mathbb{X} = \mathcal{M}$,
\item $\mathbb{Y} = \triangle_K$,
\end{minipage}
\begin{minipage}{0.4\linewidth}
\item $\Phi(\vmu) = \triangle_K$,
\item $u(\vmu, \w) = \max_{i\in \mathcal{I}} D(\w, \vmu, \neg i)$.
\end{minipage}
\end{itemize}
$\Phi$ is compact-valued, non-empty and continuous (since it is constant). $u$ is continuous since it is a finite maximum of continuous functions. The hypotheses are verified and the theorem gives the wanted result.
\end{proof}

\begin{lemma}
$i_F(\vmu)$ is upper hemicontinuous and compact valued.
\end{lemma}
\begin{proof}
We apply Theorem~\ref{th:berge} to
\begin{itemize}
\begin{minipage}{0.4\linewidth}
\item $\mathbb{X} = \mathcal{M}$,
\item $\mathbb{Y} = \{1, \ldots, K\}$,
\end{minipage}
\begin{minipage}{0.4\linewidth}
\item $\Phi(\vmu) = \{1, \ldots, K\}$,
\item $u(\vmu, \w) = D(\vmu, \neg i)$.
\end{minipage}
\end{itemize}
$\Phi$ is compact-valued, non-empty and continuous (since it is constant). $u$ is continuous Lemma~\ref{lem:D_mu_continuous}. The hypotheses are verified and the theorem gives the wanted result.
\end{proof}

\section{Algorithm Analysis}

\subsection{Proof of Lemma \ref{lem:when_concentration_fails}}\label{sec:proof_lem_when_concentration_fails}

\begin{lemma}[Lemma 19 of \citealt{GK16}]
There exists two constants $B$ and $C$ (that depend on $\mu$ and $\xi$) such that
\begin{align*}
\pr_\mu({\mathcal{E}_T'}^c) \leq BT e^{-C\sqrt{h(T)}} \: .
\end{align*}
\end{lemma}

For $h(T) = \sqrt{T}$, $\sum_{T=1}^{+\infty}\pr_\mu({\mathcal{E}_T'}^c)$ is then finite.

\begin{lemma}[\citealt{MCP14}]
Let $\beta(t,\delta) = \log(C t^2/ \delta)$ with $C$ a constant verifying the inequality $C \geq e \sum_{t=1}^\infty (\frac{e}{K})^K \frac{(\log^2(Ct^2)\log(t))^K}{t^2}$ . Then
\begin{align*}
\pr_\vmu
	\left\{
		\exists t\in\N, \: \sum_{k=1}^K N_{t,k} d(\hat{\mu}_{t,k}, \mu_k) > \beta(t, \delta)
	\right\}
\leq \delta \: .
\end{align*}
\end{lemma}

Let $f(t) = \exp(\beta(t, 1/t^5)) = C t^{10}$ in the definition of the confidence ellipsoid $\mathcal{C}_t$ (see Algorithm~\ref{alg:sticky}). Then
\begin{align*}
\sum_{T=T_0}^{+\infty}\pr_\vmu(\mathcal{E}_T^c) \leq \sum_{T=1}^{+\infty}\sum_{t=\sqrt{T}}^T \frac{1}{t^5} < + \infty \: .
\end{align*}

\subsection{Proof of Lemma \ref{lem:optimal_pulling_implies_small_stopping_time}}\label{sec:proof_lem_optimal_pulling_implies_small_stopping_time}

Set $T>T_1$ and suppose that $\mathcal{E}_T\cap\mathcal{E}_T'$ is true. For $t\geq\eta(T)$, if $\tau_\delta > t$ then $t C^*_{\varepsilon,\xi}(\vmu) \leq D(N_t,\hat{\vmu}_t,\neg i_\vmu) \leq \beta(t,\delta)$ , hence $t\leq \beta(t,\delta)/ C^*_{\varepsilon;\xi}(\vmu)$ .
\begin{align*}
\min(\tau_\delta, T)
&\leq \lceil \eta(T) \rceil + \sum_{t=\lceil \eta(T) \rceil+1}^{T} \mathbb{I}\{\tau_\delta > t-1\} \\
&\leq \lceil \eta(T) \rceil + \sum_{t=\lceil \eta(T) \rceil+1}^{T} \mathbb{I}\{t \leq 1+\frac{\beta(t,\delta)}{C^*_{\varepsilon,\xi}(\vmu)}\} \\
&\leq \lceil \eta(T) \rceil + \sum_{t=\lceil \eta(T) \rceil+1}^{T} \mathbb{I}\{t \leq 1+\frac{\beta(T,\delta)}{C^*_{\varepsilon,\xi}(\vmu)}\}\\
&\leq \max\left( \lceil \eta(T) \rceil, 1+\frac{\beta(T,\delta)}{C^*_{\varepsilon,\xi}(\vmu)} \right) \: .
\end{align*}
Suppose that $\tau_\delta > T$. Then $T \leq \max\left( \lceil \eta(T) \rceil, 1+\frac{\beta(T,\delta)}{C^*_{\varepsilon,\xi}(\vmu)} \right)$ . But by hypothesis, $\eta(T) < T-1$, such that that inequality implies $T \leq 1+\frac{\beta(T,\delta)}{C^*_{\varepsilon,\xi}(\vmu)} $ . We conclude that
\begin{align*}
\tau_\delta \leq \inf\left\{T: T>1+\frac{\beta(T,\delta)}{C^*_{\varepsilon,\xi}(\vmu)} \right\} \: .
\end{align*}

\subsection{Continuity Results and Proof of Lemma~\ref{lem:average_w_converges_to_optimal_if_convex}}\label{sec:proof_lem_avegage_w_converges}

\begin{lemma}\label{lem:convex_implies_average_also_converges_to_set}
Let $\varepsilon>0$ and $A\subseteq \triangle_K$ be a convex set and let $\w_1, \ldots, \w_t \in \triangle_K$ be such that for all $s\in[t]$, $\inf_{\w\in A}\Vert \w_s - \w \Vert_\infty \leq \varepsilon$ . Then $\inf_{\w\in A}\Vert \frac{1}{t} \sum_{s=1}^t \w_s - \w \Vert_\infty \leq \varepsilon$ .
\end{lemma}
\begin{proof}
For $s\in[t]$, let $\w^*_s \in A$ be such that $\Vert \w_s - \w^*_s \Vert_\infty \leq \varepsilon$ . Then $\frac{1}{t} \sum_{s=1}^t \w^*_s \in A$ by convexity and
\begin{align*}
\Vert \frac{1}{t} \sum_{s=1}^t \w_s - \frac{1}{t} \sum_{s=1}^t \w^*_s \Vert_\infty
&\leq \frac{1}{t} \sum_{s=1}^t \Vert \w_s - \w^*_s\Vert_\infty
\leq \varepsilon \: .
\end{align*}
\end{proof}

\begin{proof}[Proof of Lemma~\ref{lem:average_w_converges_to_optimal_if_convex}]
Let $\varepsilon>0$. By Theorem~\ref{th:continuity}, $\w^*$ is upper hemicontinuous: there exists $\xi>0$ such that if $\Vert \hat{\vmu}_t - \vmu\Vert_\infty \leq \xi$ then for all $\w_t \in \w^*(\hat{\vmu}_t)$, $\inf_{\w \in \w^*(\vmu)} \Vert\w_t - \w\Vert_\infty\leq \varepsilon$ .

For $\xi>0$, there exists $t_\xi$ such that for $t\geq t_\xi$, $\Vert \hat{\vmu}_t - \vmu\Vert_\infty \leq \xi$ by hypothesis. Then for $\w \in \w^*(\vmu)$,
\begin{align*}
\left\Vert \frac{1}{t} \sum_{s=1}^t \w_s - \w \right\Vert_\infty &\leq \frac{t_\xi}{t} + \frac{t-t_\xi}{t}\left\Vert \frac{1}{t-t_\xi} \sum_{s=t_\xi}^t \w_s - \w \right\Vert_\infty \: .
\end{align*}
Taking infimums and using the convexity of $\w^*(\vmu)$ to apply Lemma~\ref{lem:convex_implies_average_also_converges_to_set},
\begin{align*}
\inf_{\w \in \w^*(\vmu)}\left\Vert \frac{1}{t} \sum_{s=1}^t \w_s - \w \right\Vert_\infty
&\leq\frac{t_\xi}{t} + \varepsilon \leq 2\varepsilon \mbox{ for $t\geq t_\xi/\varepsilon$ .}
\end{align*}
\end{proof}

\subsection{Proof of Lemma $\ref{lem:convergence_oracle_answer}$}\label{sec:proof_lem_convergence_oracle_answer}

Under $\mathcal{E}_T$, for $t\geq h(T)$, $\vmu \in \mathcal{C}_t$. Hence $i_F(\vmu) \subseteq I_t$.

For $\vmu, \vmu' \in \mathcal{M}$, let $ch(\vmu, \vmu') = \inf_{\vlambda\in\R^K} \sum_{k=1}^K (d(\lambda_k, \mu_k) + d(\lambda_k, \mu_k'))$. $ch$ is a semi-metric on $\mathcal{M}$.

Suppose that the event $\mathcal{E}_T$ holds and set $t> h(T)$. Then on one hand $\vmu \in \mathcal{C}_t$, such that $D(N_{t-1}, \hat{\vmu}_{t-1}, \vmu) \leq \log f(t-1)$. On the other hand, by definition every point $\vmu'\in\mathcal{C}_t$ verifies $D(N_{t-1}, \hat{\vmu}_{t-1}, \vmu') \leq \log f(t-1)$. We obtain in particular that
\begin{align*}
\sum_{k=1}^K N_{t-1,k}(d(\hat{\mu}_{t-1,k}, \mu_k) + d(\hat{\mu}_{t-1,k}, \mu_k')) \leq 2 \log f(t-1) \: .
\end{align*}
By hypotheses, $N_{t-1,k} \geq n(t-1)$. We obtain that $ch(\vmu, \vmu') \leq \frac{2\log f(t-1)}{n(t-1)}$ for all $\vmu'\in\mathcal{C}_t$ .

By upper hemicontinuity of $i_F(\vmu)$, there exists $\varepsilon>0$ such that $\Vert \vmu - \vmu'\Vert_\infty \leq \varepsilon \Rightarrow i_F(\vmu') \subseteq i_F(\vmu)$. There exists $\Delta>0$ such that $ch(\vmu, \vmu') \leq \Delta \Rightarrow i_F(\vmu') \subseteq i_F(\vmu)$ .

For such a $\Delta$ and $T_\Delta = \inf\{t\in\N \: : \: \frac{2\log f(t)}{n(t)} \leq \Delta\}$, if $t\geq \max(h(T), T_\Delta)$, then $i_F(\vmu') \subseteq i_F(\vmu)$ for all $\vmu'\in\mathcal{C}_t$. hence $I_t = \bigcup_{\vmu'\in\mathcal{C}_t}i_F(\vmu') \subseteq i_F(\vmu)$ .

\subsection{Proof of Lemma $\ref{lem:N_t_converges_to_w*}$}\label{sec:proof_lem_N_t_converges_to_w*}

\begin{lemma}[\citealt{GK16}]\label{lem:tracking}
For all $t\geq 1$ and $k\in[K]$, the C-tracking rule ensures that $N_{t,k} \geq \sqrt{t + K^2} - 2K$ and that
\begin{align*}
\left\Vert N_t - \sum_{s=0}^{t-1} \w_s \right\Vert_\infty \leq K(1+\sqrt{t}) \: .
\end{align*}
\end{lemma}

\begin{lemma}\label{lem:tracking_to_optimality}
Suppose that there exists $T_I\in\N$ such that for $T\geq T_I$, $\w_t \in \w^*(\hat{\vmu}_t, \neg i_\vmu)$. Then for $T$ such that $h(T)\geq T_I$ , it holds that on $\mathcal{E}_T'$ C-Tracking verifies
\begin{align*}
\forall t\geq 4\frac{K^2}{\varepsilon^2} + 3\frac{h(T)}{\varepsilon}, \: \inf_{\w\in \w^*(\vmu, \neg i_\vmu)} \Vert \frac{N(t)}{t} - \w\Vert_\infty \leq 3\varepsilon \: .
\end{align*}
\end{lemma}

\begin{proof}
Suppose that $T$ verifies $h(T)\geq T_I$. Using Lemma~\ref{lem:tracking}, for $t\geq h(T)$ one can write for all $\w\in \w^*(\vmu, \neg i_\vmu)$,
\begin{align*}
\left\Vert \frac{N_t}{t} - \w \right\Vert_\infty
&\leq \left\Vert \frac{N_t}{t} - \frac{1}{t}\sum_{s=0}^{t-1} \w_s \right\Vert_\infty
	+ \left\Vert \frac{1}{t}\sum_{s=0}^{t-1} \w_s - \w \right\Vert_\infty \\
&\leq \frac{K(1+\sqrt{t})}{t} + \left\Vert \frac{1}{t}\sum_{s=0}^{t-1} \w_s - \w \right\Vert_\infty\\
&\leq \frac{2K}{\sqrt{t}} + \frac{h(T)}{t} + \left\Vert \frac{1}{t}\sum_{s=h(T)}^{t-1} (\w_s - \w) \right\Vert_\infty \: .
\end{align*}

The definition of event $\mathcal{E}_T'$ uses $\xi>0$ such that if $\Vert \hat{\vmu}_t - \vmu \Vert_\infty \leq \xi$ then for all $\w_t \in \w^*(\hat{\vmu}_t, \neg i_\vmu)$, $\inf_{\w \in \w^*(\vmu, \neg i_\vmu)}\Vert \w_t - \w \Vert_\infty \leq \varepsilon$ .
Under that event, for $t\geq h(T)$,
\begin{align*}
\Vert \hat{\vmu}_t - \vmu \Vert_\infty &\leq \xi \: ,\\
\forall \w_t \in \w^*(\hat{\vmu}_t, \neg i_\vmu), \: \inf_{\w \in \w^*(\vmu, \neg i_\vmu)}\Vert \w_t - \w \Vert_\infty &\leq \varepsilon \: .
\end{align*}
The convexity of $\w^*(\vmu, \neg i_\vmu)$ ensures that $\inf_{\w \in \w^*(\vmu, \neg i_\vmu)}\Vert \frac{1}{t}\sum_{s=h(T)}^T \w_s - \w\Vert_\infty \leq \varepsilon$ as well by Lemma~\ref{lem:convex_implies_average_also_converges_to_set}.
Hence, taking infimums and using the hypothesis that the event $\mathcal{E}_T'$ holds,
\begin{align*}
\inf_{\w\in\w^*(\vmu, \neg i_\vmu)} \left\Vert \frac{N_t}{t} - \w \right\Vert_\infty
&\leq \frac{2K}{\sqrt{t}} + \frac{h(T)}{t} + \inf_{\w\in\w^*(\vmu, \neg i_\vmu)} \left\Vert \frac{1}{t}\sum_{s=h(T)}^{t-1} (\w_s - \w) \right\Vert_\infty\\
&\leq \frac{2K}{\sqrt{t}} + \frac{h(T)}{t} + \varepsilon \: .
\end{align*}
For $t\geq 2\frac{K^2}{\varepsilon^2}\left( 1 + \varepsilon\frac{h(T)}{2K^2} + \sqrt{1 + \varepsilon\frac{h(T)}{K^2}} \right)$, the right-hand-side is smaller than $2\varepsilon$. In particular, this is also true for $t\geq 4\frac{K^2}{\varepsilon^2} + 3\frac{h(T)}{\varepsilon} $ .
\end{proof}

\begin{proof}[Proof of Lemma~\ref{lem:N_t_converges_to_w*}]
Let $T\in\N$ be such that $h(T) \geq T_\Delta$. Under $\mathcal{E}_T$, for $t\geq h(T)$ the set $I_t$ is constant and equal to $i_F(\vmu)$ by Lemma~\ref{lem:convergence_oracle_answer}. For this stage on, $i_t = i_\vmu$ is constant and $\w^*(\vmu, \neg i_t) \subseteq \w^*(\vmu)$. We can hence take $T_I = h(T)$ in Lemma~\ref{lem:tracking_to_optimality} and we get the wanted result.
\end{proof}

\subsection{Proof of the empirical complexity of Track and Stop}\label{sec:proof_TaS_sample_complexity}

We prove Theorem~\ref{th:tas_sample_complexity} and Theorem~\ref{th:tas_sample_complexity_general}.

Lemma~\ref{lem:convergence_oracle_answer} depends only on the amount of forced exploration, thus it is valid for Track and Stop. After some $T_\Delta > 0$, $I_t = i_F(\vmu)$. Since $\hat{\vmu}_t \in I_t$, $i_F(\hat{\vmu}_t) \subseteq i_F(\vmu)$ . The alternative selected by Track and Stop to compute $\w_t$ will be in $i_F(\vmu)$.

We modify the event $\mathcal{E}_T'$ used for Sticky Track and Stop into $\mathcal{E}_T' = \bigcap_{t=h(T)}^T \{\Vert \hat{\vmu}_t - \vmu\Vert_\infty \leq \xi\}$ where $\xi$ is such that
\begin{align*}
\Vert \vmu' - \vmu\Vert_\infty \leq \xi \Rightarrow \forall \w' \in \w^*(\vmu')\: \exists \w \in \w^*(\vmu), \: \Vert \w' - \w \Vert_\infty \leq \varepsilon \: .
\end{align*}
The difference is that we use the upper hemicontinuity of $\w^*(\vmu)$ instead of $\w^*(\vmu, \neg i)$ for some $i\in\mathcal{I}$.

From that point on, we proceed as in the proof of the sample complexity of Sticky Track and Stop, except that $N_t/t$ will not necessarily converge to $\w^*(\vmu, i_\vmu)$,  but to $\conv(\w^*(\vmu))$, convex hull of $\w^*(\vmu)$. Lemma~\ref{lem:N_t_converges_to_w*} is true for Track and Stop with the adapted $\mathcal{E}_T'$ if $\w^*(\vmu, i_\vmu)$ is replaced by $\conv(\w^*(\vmu))$. The analogue of $C^*_{\varepsilon, \xi}(\vmu)$ is
\begin{align*}
C_{\varepsilon,\xi}(\vmu) =
\inf_{
	\substack{\vmu':\Vert\vmu' - \vmu\Vert_\infty \leq \xi\\
	 \w': \inf_{\w\in \conv(\w^*(\vmu))}\Vert \w' - \w \Vert_\infty \leq 3\varepsilon}
}
D(\w',\vmu') \: . 
\end{align*}

Let $T_\Delta$ be defined as in Lemma~\ref{lem:convergence_oracle_answer}. Let $T$ be such that $h(T)\geq T_\Delta$. Let $\eta(T) = 4\frac{K^2}{\varepsilon^2} + 3\frac{h(T)}{\varepsilon}$ .
By Lemma \ref{lem:N_t_converges_to_w*} changed as explained, for $t\geq \eta(T)$, if $\mathcal{E}_T^{}\cap \mathcal{E}_T'$ then $D(N_t, \hat{\vmu}_t) \geq t C_{\varepsilon, \xi}(\vmu)$.

We now apply Lemma~\ref{lem:optimal_pulling_implies_small_stopping_time}. $\eta(T)< T-1$ if $h(T) < \frac{\varepsilon}{3}(T-1) - \frac{4}{3}\frac{K^2}{\varepsilon}$. For $h(T) = \sqrt{T}$ and $T$ bigger than a constant $T_\eta$ depending on $K$ and $\varepsilon$, this is true. Then under $\mathcal{E}_T^{}\cap \mathcal{E}_T'$, the hypotheses of Lemma~\ref{lem:optimal_pulling_implies_small_stopping_time} are verified with $T_1 = h^{-1}(\max(T_\Delta, T_\eta))$.
 
We obtain that the hypotheses of Lemma~\ref{lem:tau_delta_decomposition} are verified for
\begin{align*}
T_0
&= \max(T_1, \inf\{T: 1 + \frac{\beta(T,\delta)}{C_{\varepsilon, \xi}(\vmu)} \leq T\}) \: .
\end{align*}

Note that $\lim_{\delta \to 0} T_0 / \log(1/\delta) = 1/C_{\varepsilon, \xi}(\vmu)$. Taking $\varepsilon\to 0$ (hence $\xi\to 0$ as well), we obtain $\lim_{\delta\to 0}\frac{\ex_\vmu[\tau_\delta]}{\log(1/\delta)} = \frac{1}{\lim_{\varepsilon\to 0}C_{\varepsilon,\xi}(\vmu)}$.

Finally, $\lim_{\delta \to 0} C_{\varepsilon,\xi}(\vmu) = \inf_{\w \in \conv(\w^*(\vmu))} D(\w, \vmu)$. This proves Theorem~\ref{th:tas_sample_complexity_general}.

If $i_F(\vmu)$ is a singleton, then $\w^*(\vmu)$ is convex and  $\conv(\w^*(\vmu)) = \w^*(\vmu)$, leading to the observation that $\inf_{\w \in \conv(\w^*(\vmu))} D(\w, \vmu) = D(\vmu)$. In that case, Track-and-Stop is asymptotically optimal: Theorem~\ref{th:tas_sample_complexity} is proved.


\section{Divergences}\label{app:divergences}
An important building block in pure exploration algorithms is the largest weighted distance from $\vmu$ to the closest point $\vlambda$ in some set of alternatives,
\[
  D(\vmu, \Lambda)
  ~=~
  \sup_{\w \in \triangle} \inf_{\vlambda \in \Lambda} \sum_k w_k d(\mu_k, \lambda_k)
\]
In this section we compute a few of these distances in closed form to get a feeling for their behaviour. We do it for the Gaussian divergence $d(\mu, \lambda) = \frac{1}{2} (\mu-\lambda)^2$.

\subsection{Hyper-planes and Half-spaces}

\begin{lemma}\label{lem:line}
  When $\Lambda = \setc*{\vlambda \in \R^K}{\tuple{\a, \vlambda} = b}$ is a hyper-plane, we find
  \[
    D(\vmu, \Lambda)
    ~=~
    \frac{1}{2}
    \del*{
      \frac{
        \tuple{\a, \vmu} - b
      }{
        \sum_{i=1}^d
        |a_i|
      }
    }^2
    \qquad
    \text{and}
    \qquad
    w^*_i(\vmu)
    ~=~
    \frac{
      |a_i|
    }{
      \sum_{i=1}^d |a_i|
    }
    .
  \]
\end{lemma}
Note that the same result holds for the half-space $\Lambda = \setc*{\vlambda \in \R^K}{\tuple{\a, \vlambda} \ge b}$ when $\vmu \notin \Lambda$, i.e.\ $\tuple{\a, \vmu} < b$. If $\vmu \in \Lambda$ then $D(\vmu, \Lambda) = 0$. The Lemma implies in particular that for Best Arm Identification with $K=2$ arms, corresponding to $\a = (-1,+1)$ and $b=0$, the optimal weights $\w^*$ are uniform, as was shown by \cite{on.the.complexity}. For the $\epsilon$-BAI variant of the problem we set $b = \pm\epsilon$, so here $\w^*$ is also uniform.

\begin{proof}
We have
\[
  D(\vmu, \Lambda)
  ~=~
  \sup_{\w \in \triangle}
  \inf_{\vlambda : \tuple{\a, \vlambda} = b}~
  \sum_{i=1}^d w_i \frac{(\mu_i - \lambda_i)^2}{2}
\]
Introducing Lagrange multiplier $\theta$, we find
\[
  \sup_{\w \in \triangle}
  \sup_{\theta}
  \inf_{\vlambda \in \mathbb R^d}
  ~
  \sum_{i=1}^d w_i \frac{(\mu_i - \lambda_i)^2}{2}
  -
  \theta \del*{\tuple{\a, \vlambda} - b}
\]
Plugging in the solution
$
\lambda_i
=
\mu_i
+
\frac{
  \theta a_i
}{
  w_i
}
$
results in
\[
  \sup_{\w \in \triangle}
  \sup_{\theta}
  ~
  -
  \theta^2
  \sum_{i=1}^d
  \frac{
    a_i^2
  }{
    2 w_i
  }
  -
  \theta \del*{ \tuple{\a, \vmu} - b}
\]
Now solving for $\theta$ results in
$
\theta
=
\frac{
  - \del*{ \tuple{\a, \vmu} - b}
}{
  \sum_{i=1}^d
  \frac{
    a_i^2
  }{
    w_i
  }
}
$
and objective function value
\[
  \sup_{\w \in \triangle}
  ~
  \frac{
    \del*{ \tuple{\a, \vmu} - b}^2
  }{
    2
    \sum_{i=1}^d
    \frac{
      a_i^2
    }{
      w_i
    }
  }
\]
Further solving for $\w$ tells us that
$
w_i
~=~
\frac{
  |a_i|
}{
  \sum_{i=1}^d |a_i|
}
$
and hence the value is
\[
  \frac{1}{2}
  \del*{
    \frac{
      \tuple{\a, \vmu} - b
    }{
      \sum_{i=1}^d
      |a_i|
    }
  }^2
  .
\]
\end{proof}

\subsection{Minimum Threshold}
The following two lemmas appear as \cite[Lemma~1]{minimums} for general divergences $d(\mu, \lambda)$.
\begin{lemma}
  Let $\Lambda = \setc*{\vlambda \in \R^K}{\min_k \lambda_k \le \gamma}$. Then when $\vmu \notin \Lambda$,
  \[
    D(\vmu, \Lambda)
    ~=~
    \frac{1}{\sum_k \frac{1}{d(\mu_a, \gamma)}}
    \qquad
    \text{where}
    \qquad
    w_k^*(\vmu)
    ~=~
    \frac{
      \frac{1}{d(\mu_k, \gamma)}
    }{
      \sum_j \frac{1}{d(\mu_j, \gamma)}
    }
  \]
\end{lemma}

\begin{lemma}
  Let $\Lambda = \setc*{\vlambda \in \R^K}{\min_k \lambda_k \ge \gamma}$. Then when $\vmu \notin \Lambda$,
  \[
    D(\vmu, \Lambda)
    ~=~
    d\del[\big]{\min_k \mu_k, \gamma}
    \qquad
    \text{where}
    \qquad
    \w^*(\vmu)
    ~=~
    \mathbf 1_{k = \argmin_j \mu_j}
    .
  \]
\end{lemma}

For the version of the problem with slack $\epsilon$, we can simply replace $\gamma$ by the appropriate $\gamma \pm \epsilon$.

\subsection{Sphere}

We now consider the distance to the sphere both from within and from the outside.

\begin{lemma}
  Let $\Lambda = \setc*{\vlambda \in \R^K}{\norm{\vlambda} = 1}$. Consider any $\vmu \in \mathbb R^K$. Then
  \begin{align*}
    D(\vmu, \Lambda) 
    &= \frac{1}{2 K^2} \del*{\sqrt{K \del[\big]{1-\norm{\vmu}^2}+\norm{\vmu}_1^2}
    -
    \norm{\vmu}_1 }^2
    \\
    \text{and}
    \quad
    w^*_k(\vmu)
    &=
      \frac{1}{K}
  + \frac{\abs{\mu_k} - \frac{1}{K}\norm{\vmu}_1}{\sqrt{K\del[\big]{1- \norm{\vmu}^2}+\norm{\vmu}_1^2}}
  ,
\end{align*}
provided that
\begin{equation}\label{eq:requirement}
  \del[\Big]{\norm{\vmu}_1 - K \min_k \abs{\mu_k}}^2
  ~\le~
  \norm{\vmu}_1^2 - K \del[\big]{\norm{\vmu}^2-1}
\end{equation}
\end{lemma}
Note that the proviso is always satisfied when $\norm{\vmu} \le 1$. When $\norm{\vmu} > 1$ it depends. When the proviso is not satisfied, boundary conditions are active. In that case a pairwise swapping argument shows that $w^*_k(\vmu) = 0$ for the $k$ of minimal $\abs{\mu_k}$. The rest of the solution is found by removing $k$, and solving the remaining problem of size $K-1$.

\begin{proof}
We need to find
\[
  D(\vmu, \Lambda)
  ~=~
  \max_{\w \in \triangle} \min_{\vlambda : \norm{\vlambda} = 1}~ \frac{1}{2}\sum_k w_k (\mu_k - \lambda_k)^2
\]
As strong duality holds for the inner problem \cite[Appendix~B]{cvxbook}, we may introduce a Lagrange multiplier $\theta$ for the constraint, and write
\[
  \max_{\w \in \triangle, \theta} \min_{\vlambda \in \R^K}~ \frac{1}{2} \sum_k w_k  (\mu_k - \lambda_k)^2 + \frac{\theta}{2} \del*{1-\norm{\vlambda}^2}
  .
\]
The innermost problem is unbounded in $\vlambda$ unless $\min_k w_k \ge \theta$, so we add this as an outer constraint. Then the minimiser is found at $\lambda_k = \frac{\mu_k}{1 - \frac{\theta}{w_k}}$,
and by substituting that in, the problem simplifies to
\[
  \max_{\substack{\w \in \triangle \\ \theta \le \min_k w_k}}~ -\frac{1}{2} \sum_k  \frac{\mu_k^2}{\frac{1}{\theta} - \frac{1}{w_k}}
  + \frac{\theta}{2}
  ~=~
  \max_{\substack{\w \in \triangle \\ \theta \le \min_k w_k}}~ -\frac{1}{2}  \sum_k \frac{\mu_k^2 \theta^2}{w_k - \theta}
  + \frac{\theta}{2} \del*{1-\norm{\vmu}^2}
  .
\]
As a point of interpretation, note that we will find $\theta > 0$ when $\norm{\vmu} < 1$, and $\theta < 0$ for $\norm{\vmu} > 1$. Next we solve for $\w$, enforcing unit sum (but delaying non-negativity). With Lagrange multiplier $c$, we need to have
\[
  c ~=~
  \frac{1}{2}  \frac{\mu_k^2}{\del*{\frac{w_k}{\theta} - 1}^2}
  \qquad
  \text{resulting in}
  \qquad
  w_k
  ~=~
  \theta \del*{
    1+
    \sqrt{
      \frac{\mu_k^2}{2 c}
    }
  }
  .
\]
Solving for the normalisation results in
\[
  c
  ~=~
  \frac{1}{2}
  \del*{
    \frac{\theta}{1 - \theta K} \norm{\vmu}_1}^2
  \qquad
  \text{whence}
  \qquad
  w_k
  ~=~
  \theta +
  \del*{1 - \theta K}
  \frac{\abs{\mu_k}}{
    \norm{\vmu}_1
  }
  .
\]
Plugging this in, it remains to solve
\[
  \max_{\theta}~ -\frac{\theta^2}{2 \del*{1 - \theta K}}
  \norm{\vmu}_1^2
  + \frac{\theta}{2} \del*{1-\norm{\vmu}^2}
  .
\]
This concave problem is bounded by non-negativity of the right-hand side of \eqref{eq:requirement}. Cancelling the derivative results in a quadratic equation, with the single feasible solution
\[
  \theta
  ~=~
  \frac{1}{K}
  \del*{1-\frac{\norm{\vmu}_1}{\sqrt{K\del*{1- \norm{\vmu}^2}+\norm{\vmu}_1^2}}}
  .
\]
Filling this in yields the value and weights of the Lemma. Finally, we need to check for negativity in the weights. Using the above expression for the weights, have $\min_k w_k \ge 0$ if
\[
  - \frac{
    1
  }{
    \frac{\norm{\vmu}_1}{\min_k \abs{\mu_k}}
    - K
  }
  ~\le~
  \theta
  \qquad
  \text{i.e.}
  \qquad
  \frac{\norm{\vmu}_1}{\sqrt{\norm{\vmu}_1^2 - K \del*{\norm{\vmu}^2-1}}}
  ~\le~
  \frac{
    1
  }{
    1 - K \frac{\min_k \abs{\mu_k}}{\norm{\vmu}_1}
  }
\]
which we can further reorganise to \eqref{eq:requirement}, as required.
\end{proof}

\subsection{Composition of two independent problems}

We consider the case where we seek to answer two independent queries on disjoint sets of arms. Let $A,B$ be a partition of $[K]$. Suppose that the structure of the problem and the answers decompose according to this partition, i.e. $\mathcal M = \mathcal{M}^A \times \mathcal{M}^B$, $\mathcal I = \mathcal{I}^A\times\mathcal{I^B}$ and $i^*(\vmu) = i_A^*(\vmu^A) \times i^*_B(\vmu^B)$. Then we can also write all alternative sets $\neg i$ as $\neg i_A \times \mathcal{I}^B \cup \mathcal{I}^A \times \neg i_B$. It then holds that for all $\w \in \w^*(\vmu, \neg i)$,
\begin{align*}
\sum_{k\in A} w_k &= \frac{\frac{1}{D(\vmu^A, \neg i_A)}}{\frac{1}{D(\vmu^A, \neg i_A)}+\frac{1}{D(\vmu^B, \neg i_B)}}\: ,\\
\sum_{k\in B} w_k &= \frac{\frac{1}{D(\vmu^B, \neg i_B)}}{\frac{1}{D(\vmu^A, \neg i_A)}+\frac{1}{D(\vmu^B, \neg i_B)}}\: ,\\
\frac{1}{D(\vmu, \neg i)} &= \frac{1}{D(\vmu^A, \neg i_A)}+\frac{1}{D(\vmu^B, \neg i_B)} \: .
\end{align*}
Since the sample complexity is proportional to $1/D$ we obtain the natural conclusion that the number of samples needed to solve two independent queries is the sum of the samples needed by each query.

\end{document}